\documentclass[10pt,journal,compsoc]{IEEEtran}
\usepackage{times,amsmath,epsfig}
\usepackage{helvet}  
\usepackage{courier}  
\usepackage{url}  
\usepackage{graphicx}  
\usepackage{amsthm}
\usepackage{bm}
\usepackage{algorithm}
\usepackage{algorithmicx}
\usepackage{algpseudocode}
\usepackage{multirow}
\usepackage{bigstrut}
\usepackage{amssymb}
\usepackage{subfigure}
\usepackage{booktabs}
    
\ifCLASSOPTIONcompsoc
\usepackage[nocompress]{cite}
\else
\usepackage{cite}
\fi
\ifCLASSINFOpdf
\else
\fi
\begin{document}

%
\title{Unsupervised Heterogeneous Coupling Learning for Categorical Representation}
%
%
%
%

\author{Chengzhang Zhu,
    Longbing Cao,~\IEEEmembership{Senior Member,~IEEE,}~
    and~Jianping Yin

\thanks{Chengzhang Zhu and Longbing Cao (corresponding author) are with the Advanced Analytics Institute, University of Technology Sydney, Australia. 
E-mail: kevin.zhu.china@gmail.com, Longbing.Cao@gmail.com.
Jianping Yin is with Dongguan University of Technology, China.
}
}

%
%

\markboth{Journal of \LaTeX\ Class Files,~Vol.~14, No.~8, August~2015}%
{Shell \MakeLowercase{\textit{et al.}}: Bare Demo of IEEEtran.cls for Computer Society Journals}
%



\IEEEtitleabstractindextext{%

\begin{abstract}
    Complex categorical data is often hierarchically coupled with heterogeneous relationships between attributes and attribute values 
    and the couplings between objects. Such value-to-object couplings are heterogeneous with complementary and inconsistent interactions and distributions. 
    Limited research exists on unlabeled categorical data representations, ignores the heterogeneous and hierarchical couplings, underestimates data characteristics and complexities, and overuses redundant information, etc. The deep representation learning of unlabeled categorical data is challenging, overseeing such value-to-object couplings, complementarity and inconsistency, and requiring large data, disentanglement, and high computational power. This work introduces a shallow but powerful UNsupervised heTerogeneous couplIng lEarning (UNTIE) approach for representing coupled categorical data by untying the interactions between couplings and revealing heterogeneous distributions embedded in each type of couplings. UNTIE is efficiently optimized w.r.t. a kernel \textit{k}-means objective function for unsupervised representation learning of heterogeneous and hierarchical value-to-object couplings. Theoretical analysis shows that UNTIE can represent categorical data with maximal separability while effectively represent heterogeneous couplings and disclose their roles in categorical data. The UNTIE-learned representations make significant performance improvement against the state-of-the-art categorical representations and deep representation models on 25 categorical data sets with diversified characteristics.
\end{abstract}

}


\maketitle
\newtheorem{theorem}{Theorem}[section]
\theoremstyle{definition}
\newtheorem{defn}[theorem]{Definition} 
\newtheorem{lemma}{Lemma}
\renewcommand{\algorithmicrequire}{\textbf{Input:}}
\renewcommand{\algorithmicensure}{\textbf{Output:}}

\IEEEdisplaynontitleabstractindextext

%
\IEEEpeerreviewmaketitle

\IEEEraisesectionheading{\section{Introduction}\label{sec:introduction}}

\IEEEPARstart{T}he recent years have seen increasing research on learning the representation of complex categorical data ('categorical representation' for short) \cite{bengio2013representation,boriah2008similarity,bai2013impact,wang2015coupled,seth2016archetypal}, which shows critical for downstream tasks, e.g., regression \cite{tutz2016regularized}, clustering \cite{qian2016space}, classification \cite{zhu2018heterogeneous}, and outlier detection \cite{pang2016outlier}. Different from numerical data, the attribute values, attributes and objects of categorical data are often coupled with each other w.r.t. various aspects, e.g., value frequency, co-occurrence and distribution; attribute relations including correlations and dependency, interactions and hierarchy; and other data characteristics \cite{cao2015coupling,Bremaud2017,zhu2018heterogeneous,dst_Cao15}. We broadly refer them to \textit{couplings} \cite{cao2015coupling}, which are heterogeneous - diverse interactions and distributions, and hierarchical - from values to objects, and drive the complexities and dynamics of categorical data. Learning such heterogeneous and hierarchical couplings shows fundamental for appropriate categorical data representation, however, rarely explored in unsupervised settings.

Besides the critical progress made in learning the similarity and metrics of categorical data w.r.t. value co-occurrences and attribute relations \cite{ahmad2007method,ienco2012context,jia2015distance}, coupling learning \cite{cao2015coupling} explores even more comprehensive and stronger categorical representations by revealing and embedding heterogeneous value-to-object couplings on explicit attributes and latent factors \cite{wang2015coupled,zhang2015categorical,jianembedding,nips_DoC18}. Typically, categorical data is converted to either a vector \cite{jianembedding} or a similarity \cite{ng2007impact} space to leverage the missing numerical intervals between categorical values, then numerical analytical tools can be used. Such methods demonstrate a significant potential of deeply understanding intrinsic couplings in categorical data. However, rare work is available and it is very challenging to handle diverse data characteristics and complexities, including heterogeneities, interactions, structures, relations, distributions and nonstationarity, in categorical data representation \cite{cao2014non,cao2015coupling,dst_Cao15}.


A critical question raised in heterogeneous coupling learning is \textit{whether learning more couplings enhances categorical data representation}.
This issue has been initially studied by Ienco et al. \cite{ienco2012context}. They found the redundant information in various couplings may hamper the quality of categorical data representation and proposed the \textit{symmetric uncertainty} (SU) as a criterion to filter redundant couplings w.r.t. the correlations between two attributes to largely reduce the redundant information and enable better representation performance. Other recent work \cite{jianembedding,ijcai_PangCCL17} further shows that the redundant couplings can be reduced by methods like principal component analysis and shows better categorical representation performance than vector, similarity and embedding-based methods. However, rare work identifies redundant couplings and decouple them from those important ones.

Another open issue is to capture diverse interactions and relations that are complementary and inconsistent with each other while as many types of couplings are learned as possible. As illustrated by Table \ref{tab:toy}, if an intra-attribute coupling (i.e., value couplings) is measured in terms of value frequency, the difference between slightly curled and curled watermelons per the attribute \textit{root shape} is $0$ due to their same frequency. However, we can easily differentiate them (i.e. difference is not $0$) because the curled root is more related to the yellow and green watermelons while the slightly curled root is more associated with the green and black ones when the inter-attribute couplings between \textit{color} and \textit{root shape} are considered.

\begin{table}[!htpb]%
\centering
\caption{\textit{Toy Example.} The watermelon information table. Each watermelon with different sweetness is described w.r.t. three attributes: \emph{Texture}, \emph{Color}, and \emph{Root Shape}.\label{tab:toy}}%
\begin{tabular}{l|lll|l}
\toprule
\textbf{ID} & \textbf{Texture} & \textbf{Color} & \textbf{Root Shape} &  \textbf{Sweetness} \\
\midrule
A1 & clear & white & straight & low\\

A2 & blurry & yellow & straight & low\\

A3 & blurry & yellow & curled & low\\

A4 & clear & green & slightly curled & low\\

A5 & blurry & green & curled & high\\

A6 & clear & black & slightly curled & high\\
\bottomrule
\end{tabular}
\end{table}%

The heterogeneity and inconsistency between couplings \cite{ralaivola2010chromatic,cao2014non} may be caused by (1) different types of couplings corresponding to distinct interactions in data and following different data distributions; and (2) multiple distributions existing in a data set. While our earlier work in \cite{zhu2018heterogeneous} analyzes and captures the heterogeneous couplings for supervised learning, no existing methods on similarity, metric and representation learning effectively handle the above challenges in unsupervised categorical representation, which is critical yet challenging for understanding the intrinsic data complexities in unlabeled categorical data. Embedding lookup table and deep representation learning methods \cite{bengio2013representation} such as one-hot embedding, word embedding, autoencoder \cite{VincentLLBM10}, adversarial learning \cite{donahue2016adversarial}, and deep models such as wide and deep model \cite{cheng2016wide} and auto-instructor MAI \cite{JianHCL18} significantly outperform shallow methods in capturing latent features and relations. However, their common approaches and advantages are built on simplifying and equally treating input (e.g., by a one-hot encoder), involving a special modeling mechanism or structure, ignoring or disentangling complicated couplings, and a deep abstraction of large data with high computational power. They struggle in representing small yet complex unlabeled categorical data and also ignore the semantics and other diverse explicit characteristics of categorical values, attributes and objects, critical for categorical data representation and learning \cite{cao2015coupling,wang2015coupled,zhu2018heterogeneous}.     


In this paper, we build a shallow but powerful UNsupervised heTerogeneous couplIng lEarning (UNTIE, for short) approach to learning heterogeneous and hierarchical couplings that may be complementary yet inconsistent in small unlabeled categorical data with complex data characteristics. As the first attempt, UNTIE simultaneously represents (1) diverse value-to-object couplings,
(2) complex relations between heterogeneous and hierarchical couplings, and (3) heterogeneous distributions of respective couplings by unsupervised multikernel learning. Specifically, complex relations are entangled by the nonlinear mapping of various kernelized coupling functions, and the heterogeneous distributions are sensitively modeled by different kernels respectively. Instead of directly combining heterogeneous couplings, UNTIE first remodels the diverse couplings by multiple kernels to transform the various couplings-based spaces to respective kernelized representation spaces of higher dimensionality. Then, UNTIE learns both the weight of each attribute value in an individual kernel space and the weights of the learned kernel spaces to reflect both the heterogeneous data distributions of respective couplings and the interactions between couplings. Further, to efficiently learn heterogeneous couplings, UNTIE seamlessly wraps the kernel space weights by a positive semi-definite kernel and optimizes this kernel by optimizing an unsupervised kernel k-means objective. Lastly, the optimized kernel is used as the similarity representation of categorical data to generate a vector representation by further decomposing this kernel. We provide theoretical analysis (see Theorem \ref{thm:cut}) to show UNTIE can represent categorical data with maximizing the separability for further learning tasks. 

Accordingly, this work delivers the following significant contributions to categorical representation, unsupervised representation, and coupling learning.
\begin{itemize}
    \item UNTIE is the first unsupervised categorical representation method to learn various value-to-object couplings and their complementarity and inconsistency. UNTIE collectively captures heterogeneous data distributions of diverse couplings and adaptively integrates the couplings by involving their interactions.
    \item UNTIE maps the heterogeneous intra- and inter-attribute couplings into multiple kernels to capture the coupling heterogeneity (\S \ref{subsec:heterogeneity}). In the kernel spaces, learning heterogeneous couplings is formalized as an efficient unsupervised optimization problem by optimizing an UNTIE-enabled kernel \textit{k}-means objective (\S \ref{subsec:unsupervised}).
    \item UNTIE works in a completely unsupervised fashion to capture the intrinsic data characteristics in categorical data with both theoretical and experimental verification. Theoretical analysis is provided that shows the UNTIE-represented data has the minimum normalized cut and increases data separability (\S \ref{sec:theory}).
\end{itemize}

We substantially verify the UNTIE effectiveness, representation quality, efficiency, flexibility and stability on 25 real-life categorical data sets with diversified data characteristics (including multi-dimensional, multi-class and multi-valued objects) and four synthetic data sets generated per a variety of data factors. UNTIE is compared with vector, similarity, embedding and deep representation methods. (1) UNTIE can effectively address both complementarity and inconsistency in learning heterogeneous couplings. (2) UNTIE enjoys accuracy gain (up to 51.72\% in terms of F-score on these data sets) from the learned heterogeneity and produces substantially better representation performance than the state-of-the-art shallow and deep categorical representation methods. (3) The efficiency of UNTIE is insensitive to the volume of data, which indicates UNTIE is scalable for large data. (4) The learned UNTIE representations can enhance different downstream learning tasks. This work also shows that shallow learning does not lose the ground under handling complex (small or large) data to deep models, particularly for unsupervised settings. 



\section{Related Work} \label{Related Work}


The quality of categorical data representations affects the performance of representations-based learning tasks. Categorical representations are determined by how well a representer captures the various value-to-object coupling relationships and their heterogeneities within and between categorical values, attributes and objects \cite{cao2015coupling,ng2007impact,boriah2008similarity}. For example, 
embedding methods like one-hot embedding and word embedding only encode the existence of a value or the IDF-based textual vector to 
a vector space. 
Matching-based methods 
treat categorical values equally and overlook their rich differences. 

A recent effort made for categorical representation is the coupling learning of complex interactions and relations \cite{cao2015coupling} and demonstrates great potential in (1) intra-attribute couplings-based representations \cite{ng2007impact,cao2012dissimilarity} and (2) inter-attribute couplings-based representations \cite{le2005association,ahmad2007method,boriah2008similarity,jia2015distance}. The former reveals the way and degree that values are coupled within an attribute. For example, the method in \cite{ng2007impact} adopts the conditional probability of the attribute values of an object w.r.t. the attribute cluster centers to represent categorical data, and the method in \cite{cao2012dissimilarity} introduces set theory for measuring intra-attribute value similarity to represent categorical data. The latter captures the way and degree that attributes are coupled. They typically measure the inter-attribute couplings w.r.t. the conditional probabilities \cite{le2005association,ahmad2007method} or co-occurrence frequencies \cite{jia2015distance} between values of different attributes. These two groups of representations outperform other classic methods such as matching-based as they capture richer interactions in categorical data. However, most of such work only considers a single type of couplings and overlooks many other characteristics in categorical data. 

The work in \cite{zhang2015categorical} shows that representing more couplings in categorical data may significantly upgrade learning performance. However, other recent research also shows that capturing more but duplicated couplings does not guarantee better categorical representation as in \cite{ienco2012context}. 
A symmetric uncertainty (SU) criterion enables better representation performance in several categorical data representation methods with multiple couplings \cite{ienco2012context,wang2013coupled,wang2015coupled}. Alternatively, the method in \cite{jianembedding} uses the principal component analysis to reduce the redundancy between couplings. These methods make performance gain by reducing redundant couplings but ignore the inconsistency between heterogeneous couplings with diverse interactions and distributions, i.e., heterogeneity \cite{cao2014non} of different couplings. This issue was studied in \cite{zhu2018heterogeneous} by capturing hierarchical couplings to enhance categorical data representation with label information. None of existing unsupervised categorical representation methods explicitly and effectively model heterogeneous couplings, which brings about significant challenge to representation learning, as explored in this work. 

Deep representation learning presents increasingly promising power in representing images, text, networks, etc. \cite{bengio2013representation}. However, unsupervised categorical representation learning has not been well explored and presents a challenge to deep learning. Existing methods typically convert categorical input into a vector space through encoding such as one-hot encoding and word embedding and then rely on a deep neural architecture such as autoencoder \cite{VincentLLBM10}, adversarial learning (e.g., BiGAN and GAN variants, wide and deep network, and variational autoencoder) \cite{donahue2016adversarial,cheng2016wide,kingma2013auto} or auto-instructor \cite{jianembedding,JianHCL18} to learn the hidden relations and features. Such methods overlook or simplify most of data characteristics of categorical data (e.g., value semantics, frequencies, attribute interactions, distributions, etc.), decouple and disentangle couplings \cite{bengio2013representation}, and rely on high computational power and large (partially labeled) data. They are troubled by small, unlabeled, and complicatedly coupled categorical data. 

\section{Preliminaries}\label{sec:preliminary}

Assume a categorical data set drawn from distributions $\Phi$ can be represented as a three-element tuple $C = <O, A, V>$, where $O = \{\mathsf{o}_i | i \in N_o\}$ is an object set with $n_o$ objects; $A = \{\mathsf{a}_i | i \in N_a \}$ is an attribute set with $n_a$ attributes; and $V = \bigcup_{j=1}^{n_a}V^{(j)}$ is the collection of attribute values with $n_v$ values, in which $V^{(j)} = \{\mathsf{v}_i^{(j)} | i \in N_v^{(j)}\}$ is the set of attribute values $\mathsf{v}_i^{(j)}$ with $n_v^{(j)}$ values of attribute $\mathsf{a}_j$. $N_o$, $N_a$ and $N_v^{(j)}$ are the sets of indices for objects, attributes, and values of the $j$-th attribute, respectively. For the $i$-th object $\mathsf{o}_i$, the categorical value in the $j$-th attribute $\mathsf{a}_j$ can be represented as $v_i^{(j)}$.  
The main notations in this paper are defined in Table \ref{tab:notationList} \footnote{The specifications for symbol styles in this paper are as follows. Element: lowercase with Sans Serif font; value: lowercase; vector: lowercase with bold font; matrix: uppercase with bold font; set: uppercase; function: lowercase with parentheses; space: uppercase with Calligraphic font; value index: subscript; attribute index: superscript with parenthesis.}. 
For example, in Table \ref{tab:toy}, $O$ = \{A1, A2, A3, A4, A5, A6\}, $A$ = \{\text{Texture, Color, Root Shape}\}, $V$ = \{clear, blurry, white, yellow, green, black, straight, curled, slightly curled\}, $V^{(1)}$ = \{clear, blurry\}, and $v_2^{(3)}$ = straight.

	\begin{table}[!htpb]
		\caption{List of Notations \label{tab:notationList}}
		\begin{tabular}{l|l}
			\toprule
			Symbol & Meaning \\
			\midrule
			$\Phi$ & Categorical data distributions \\
			$C$ & Categorical data tuple\\
			$O$ & Object set \\
			$A$ & Attribute set \\
			$V$ & Categorical value set of all attributes\\
			$V^{(j)}$ & Categorical value set of the $j$-th attribute\\
			$\mathsf{o}_i$ & The $i$-th object in $O$\\
			$\mathsf{a}_i$ & The $i$-th attribute in $A$\\
			$\mathsf{v}_i$ & The $i$-th categorical value in $V$\\
			$v_i^{(j)}$ & The categorical value of $\mathsf{o}_i$ in $\mathsf{a}_j$\\
			$\mathsf{v}_i^{(j)}$ & The $i$-th categorical value in $V^{(j)}$\\
			$\mathsf{m}_i$ & The vector corresponding to $i$-th value in a coupling space\\
			$\mathsf{c}_k$ & The $k$-th cluster\\ 
			$\bm{\omega}$ & The heterogeneity parameter \\
			$n_o$ & The number of objects in $O$\\
			$n_a$ & The number of attributes in $A$\\
			$n_v$ & The number of categorical values in $V$ \\
			$n_v^{(j)}$ & The number of categorical values in $V^{(j)}$\\
			$n_k$ & The number of kernel matrices transformed from coupling spaces\\
			$n_{c}$ & The number of clusters\\
			$n_{oc}$ & The size of $c$-th cluster\\
			$n_{mv}$ & The maximal number of values in attributes\\
			$n_{av}$ & The average number of attribute values\\
			$n_{\bm{\omega}}$ & The number of elements in $\bm{\omega}$\\
			$n_i$ & The number of iterations\\
			$n_b$ & The training batch size\\
			$n_m^{(z)}$ & The number of couplings for the $z$-th attribute \\
			$N_o$ & The set of indices for objects in $O$ \\
			$N_a$ & The set of indices for attributes in $A$ \\
			$N_v^{(j)}$ & The set of indices for categorical values in $V^{(j)}$ \\
			$\mathcal{M}_{Ia}$ & Intra-attribute coupling spaces\\
			$\mathcal{M}_{Ie}$ & Inter-attribute coupling spaces\\
			$\mathcal{K}_p$ & The $p$-th kernel space transformed from a coupling space\\
			$\mathcal{K}_p^{\prime}$ & The heterogeneous kernel space transformed from $\mathcal{K}_p$\\
			$\mathbf{K}_p$ & The kernel matrix that spans $\mathcal{K}_p$\\
			$\mathbf{K}_p^{\prime}$ & The kernel matrix that spans $\mathcal{K}_p^{\prime}$\\
			$\mathbf{T}_p$ & The transformation matrix from $\mathcal{K}_p$ to $\mathcal{K}_p^{\prime}$\\
			$\mathbf{C}^{(z,k)}$ & The $k$-th coupling matrix of the $z$-th attribute  \\
			$\alpha_{pi}$ & The weight of the $i$-th value in the $p$-th kernel space\\
			$\beta_p$ & The weight of the $p$-th kernel space \\
			$\mathbf{S}$ & UNTIE similarity representation\\
			$\mathbf{X}$ & UNTIE vector representation\\
			
			\bottomrule
		\end{tabular}
	\end{table}

\section{The UNTIE Design}\label{sec:method}
UNTIE learns unsupervised categorical representation based on the rationale below: (1) a categorical value may belong to multiple distributions; (2) a coupling may make different contributions to different value distributions; and (3) the overall distribution of a categorical value can be described by multiple distributions. We call the above \textit{heterogeneity hypotheses}, which are theoretically supported by Theorem \ref{thm:factorize} in Section \ref{sec:hypothesis-proof}.

\subsection{The UNTIE Framework}
While this paper focuses on a specific instance of UNTIE, as shown in Fig. \ref{fig:untie}, UNTIE actually presents a framework of unsupervised categorical representation. It represents categorical data in both vector (as a vector representation) and similarity (as a kernel matrix) spaces. To reveal heterogeneous couplings,  UNTIE first converts categorical data to several coupling spaces by multiple coupling learning functions. Then, it feeds and transforms each coupling space to multiple kernel spaces. Further, it reduces the redundancy and inconsistency between heterogeneous couplings by learning the heterogeneity between couplings in the kernel spaces. Specifically, UNTIE differentiates the contributions of individual kernel spaces and reveals the kernel-sensitive distribution within each kernel space. To efficiently learn the heterogeneities in an unsupervised way, UNTIE wraps the weight of each kernel space and the weights of the values embedded in a kernel space by a wrapper kernel. It then optimizes this kernel by solving a kernel k-means objective, i.e., regularizing the objects within one cluster to be more similar to each other than to those in others. UNTIE uses the optimized wrapper kernel as the similarity representation of categorical data, and further generates the vector representation by decomposing the optimized wrapper kernel.

\begin{figure*}[t]
 \centering
\includegraphics[width=0.8\textwidth]{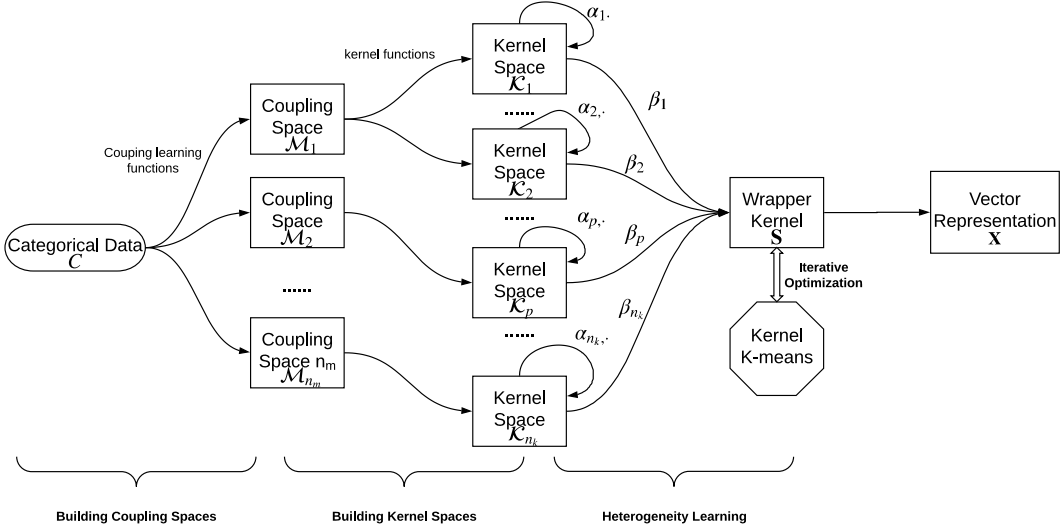}
 \caption{The UNTIE Framework. It first transforms the coupling spaces to multiple kernel spaces and then learns the heterogeneities within and between couplings in these kernel spaces by solving a kernel k-means objective.}
 \label{fig:untie}
\end{figure*}

To effectively address the coupling inconsistency problem, UNTIE learns heterogeneous couplings by multiple kernels, which transforms a coupling from its original space to several kernel spaces. Since a kernel is sensitive to a distribution \cite{bucak2014multiple}, these kernel spaces reflect different value distributions. In a kernel space, a coupling is preserved if it matches the kernel-sensitive distribution while other couplings are filtered. For this, UNTIE learns the weights of values in each kernel space to effectively reveal the multiple distributions corresponding to a coupling. The multiple distributions of a categorical value jointly contribute to the overall distribution of the value. Therefore, the learned weights of these kernel spaces filter the redundant information but integrate the complementary information.

To learn the heterogeneous couplings in an unsupervised manner, without loss of generality, UNTIE takes the assumption that the objects within one cluster are more similar to each other than to those in other clusters. This assumption is commonly taken in most of clustering and classification methods and shows its validity with real data distributions. Accordingly, UNTIE learns heterogeneous couplings for categorical data representation in an iterative way. In each iteration, UNTIE first analyzes the clusters based on its generated representation, and then tunes the representation based on the obtained clusters. To efficiently cluster data, UNTIE wraps the weight of each kernel space and the weight of the values embedded in kernel spaces by a wrapper kernel and optimizes this kernel by solving a kernel k-means objective. In addition to efficiency, the kernel k-means objective also brings other benefits for categorical data representation, such as good separability, which will be discussed in  Section \ref{sec:normalized-cut}.


\subsection{Coupling Learning}
As discussed in Introduction, we broadly refer \textit{couplings} to any interactions and relations within and between values, attributes, and objects \cite{cao2015coupling,wang2015coupled,zhu2018heterogeneous}. In categorical data, possible types of couplings include \textit{intra-attribute couplings} (i.e., value couplings, such as per value frequency, co-occurrence and matching), inter-attribute couplings (e.g., attribute correlation and dependency and unknown linkage), and object couplings built on the value and attribute couplings. UNTIE learns the value-to-attribute-to-object hierarchical couplings after learning and fusing various intra- and inter-attribute couplings in unlabeled categorical data. 


\textbf{Learning Intra-attribute Couplings.}
\textit{Intra-attribute couplings} represent the interactions between the values of an attribute and the value distributions in an attribute \cite{boriah2008similarity,wang2015coupledMixed}.
We measure intra-attribute couplings in terms of intra-attribute distributions by a value frequency function and calculate the Euclidean distance in a numerical space. Although the value frequency function has only one input value, it measures the value distribution against all values. For a categorical value $\mathsf{v}_i^{(j)}$ in the $j$-th attribute, the value frequency function $m_{Ia}^{(j)}(\mathsf{v}_i^{(j)})$ maps an intra-attribute coupling between this value and the other categorical values in this attribute to a one-dimensional intra-attribute coupling vector $ m_{Ia}^{(j)}(\mathsf{v}_i^{(j)})$.
\begin{equation}
\label{eq:intra}
    m_{Ia}^{(j)}(\mathsf{v}_i^{(j)}) = [\frac{|g^{(j)}(\mathsf{v}_i^{(j)})|}{n_o}],
\end{equation}
where $g^{(j)}(\cdot): V^{(j)}\rightarrow O$ maps the value $\mathsf{v}_i^{(j)}$  to a set of objects that have value $\mathsf{v}_i^{(j)}$ in the $j$-th attribute, $n_o$ is the number of objects, and $|\cdot|$ refers to the count of a set. For example, in Table \ref{tab:toy}, a relationship between value \textit{yellow} and object attribute \textit{color} is $g^{(j)}(yellow) = \{A2, A3\}$. The intra-attribute coupling vector of value \textit{yellow} is $m_{Ia}^{(2)}(yellow) = [\frac{|\{A2, A3\}|}{6}] = [\frac{1}{3}]$. 

An intra-attribute coupling space $\mathcal{M}_{Ia}^{(j)}$ is spanned by the intra-attribute coupling vectors obtained in an attribute by Eq. (\ref{eq:intra}) and is defined below:
\begin{equation}\label{eq:intra-couplings}
    \mathcal{M}_{Ia}^{(j)} = \{m_{Ia}^{(j)}(\mathsf{v}_i^{(j)})| \mathsf{v}_i^{(j)} \in V_j\}.
\end{equation}
For categorical data with $n_a$ attributes, the intra-attribute coupling spaces are $\mathcal{M}_{Ia} = \{\mathcal{M}_{Ia}^{(1)}, \cdots, \mathcal{M}_{Ia}^{(n_a)}\}$.
The intra-attribute coupling spaces only present a one-dimensional embedding of the categorical data space w.r.t. each attribute. The following inter-attribute couplings consider the interactions between attributes. 

\textbf{Learning Inter-attribute Couplings.}
\textit{Inter-attribute couplings} refer to the interactions between attributes and the contextual (and/or semantic) information of attribute values w.r.t. other attributes \cite{ienco2012context,wang2011coupled,cao2015coupling}. This attribute-based interactive and contextual information complements the value distributions and interactions captured by intra-attribute couplings. 
For example, in Table \ref{tab:toy}, white and black watermelons have the same frequency but can be distinguished by involving their \textit{root shapes} which are significantly different.  

Here, the inter-attribute couplings are represented by the information conditional probability, which reveals the distributions of an attribute value in the spaces spanned by the values of the other attributes. Given a value $\mathsf{v}^{(j)}$ of attribute $\mathsf{a}_j$ and a value $\mathsf{v}^{(k)}$ of attribute $\mathsf{a}_k$, the information conditional probability function is defined as follows:
\begin{equation}\label{eq:ICPF-Attri}
p(\mathsf{v}^{(j)}|\mathsf{v}^{(k)}) = \frac{|g^{(j)}(\mathsf{v}^{(j)})\cap g^{(k)}(\mathsf{v}^{(k)})|}{|g^{(k)}(\mathsf{v}^{(k)})|},
\end{equation}
where $\cap$ returns the intersection of two sets. Based on the information conditional probability function, the inter-attribute coupling learning function $m_{Ie}^{(j)}(\mathsf{v_i}^{(j)})$ embeds interactions between value $\mathsf{v_i}^{(j)}$ and other attributes as a $|V_*|$-dimensional inter-attribute coupling vector $m_{Ie}^{(j)}(\mathsf{v}_i^{(j)})$,
\begin{equation}
m_{Ie}^{(j)}(\mathsf{v}_i^{(j)}) = \left[
\begin{array}{cccc}
p(\mathsf{v}_i^{(j)}|\mathsf{v}_{*1}), & \cdots , & p(\mathsf{v}_i^{(j)}|\mathsf{v}_{*|V_*|})
\end{array}
\right]^{\top},
\label{eq:inter}
\end{equation}
where $V_* = \{V^{(k)}| k \in N_a, k \neq j\}$ is the set of values in all attributes except $\mathsf{a}_j$, and $\mathsf{v}_{*i} \in V_*$ is a categorical value in set $V_*$.
For example, the information condition probability between the yellow watermelons and those with curled root shape is $p(yellow|curled) = \frac{|\{A2, A3\} \cap \{A3, A5\}|}{|\{A3, A5\}|} = \frac{|\{A3\}|}{|\{A3, A5\}|} = \frac{1}{2}$. The inter-attribute coupling vector of value \textit{yellow} is calculated as 
\begin{equation*}
\begin{aligned}
        &m_{Ie}^{(2)}(yellow) \\
        & = [p(yellow|clear), \cdots, p(yellow|slightly~curled)] \\
        &= [0, \frac{2}{3}, \frac{1}{2}, \frac{1}{2}, 0]
\end{aligned}
\end{equation*}

An inter-attribute coupling space $\mathcal{M}_{Ie}^{(j)}$ is spanned by the inter-attribute coupling vectors obtained in an attribute by Eq. \eqref{eq:inter}:
\begin{equation}\label{eq:inter-coupling}
\mathcal{M}_{Ie}^{(j)} = \{m_{Ie}^{(j)}(\mathsf{v}_i^{(j)})|\mathsf{v}_i^{(j)}\in V^{(j)}\}.
\end{equation}
For categorical data with $n_a$ attributes, the inter-attribute coupling spaces $\mathcal{M}_{Ie} = \{ \mathcal{M}_{Ie}^{(1)}, \cdots, \mathcal{M}_{Ie}^{(n_a)} \}$.

An inter-attribute coupling learning function projects categorical values into a higher dimensional space if $|V| > 2 |V^{(j)}| - 1$, because the dimensionality of inter-attribute coupling space equals $|V| - |V^{(j)}|$, while the degree of freedom (equivalent to the dimensionality of transforming a categorical value to a dummy variable) of the $j$-th attribute is $|V^{(j)}|-1$. In this way, the value couplings incurred by other attributes are captured, which complements with the intra-attribute couplings to form a complete representation of categorical attribute space.

\subsection{Heterogeneity Learning in Kernel Spaces}\label{subsec:heterogeneity}
With the coupling spaces built above from the intra- and inter-attribute perspectives, UNTIE further constructs an entire coupling space $\mathcal{M}$ which is a collection of heterogeneous intra-attribute coupling spaces $\mathcal{M}_{Ia}$ and inter-attribute coupling spaces $\mathcal{M}_{Ie}$,
\begin{equation}
    \mathcal{M} = \mathcal{M}_{Ia} \cup \mathcal{M}_{Ie}.
\end{equation}

To effectively integrate the heterogeneous couplings in the learned coupling space set, UNTIE transforms the learned heterogeneous coupling spaces to uniform spaces, in which heterogeneous couplings are comparable. Specifically, UNTIE uses multiple kernels to transform each coupling space into its corresponding kernel spaces, where each kernel space corresponds to the transformed coupling space w.r.t. a particular kernel mapping function. It generates a set of $n_k$ ($n_k= |\mathcal{M}| \times |F|$, where $F$ is the set of kernel functions for the transformation) kernel spaces $\{\mathcal{K}_1, \mathcal{K}_2, \cdots, \mathcal{K}_{n_k} \}$, and the $p$-th ($p \leq n_k$) space is spanned by a kernel matrix $\mathbf{K}_p$, which is constructed from a coupling space $\mathcal{M}_j$ by a kernel function $k_p(\cdot,\cdot)$ for an attribute. Denoting $\mathsf{m}_i$ as a vector in $\mathcal{M}_j$ corresponding to the $i$-th categorical value, $\mathbf{K}_p$ is represented as follows,
\begin{equation}\label{eq:kernelspace}
\mathbf{K}_p = \left[
\begin{matrix}
k_p(\mathsf{m}_1, \mathsf{m}_1)& k_p(\mathsf{m}_1, \mathsf{m}_2)&\cdots& k_p(\mathsf{m}_1, \mathsf{m}_{n_v^*}) \\
k_p(\mathsf{m}_2, \mathsf{m}_1)& k_p(\mathsf{m}_2, \mathsf{m}_2)&\cdots& k_p(\mathsf{m}_2, \mathsf{m}_{n_v^*}) \\
\vdots & \vdots & \ddots & \vdots \\
k_p(\mathsf{m}_{n_v^*}, \mathsf{m}_1)& k_p(\mathsf{m}_{n_v^*}, \mathsf{m}_2)&\cdots& k_p(\mathsf{m}_{n_v^*}, \mathsf{m}_{n_v^*}) \\
\end{matrix}
\right],
\end{equation}
where $n_v^*$ is the number of categorical values represented by $\mathcal{M}_j$. For example, in Table \ref{tab:toy}, if $k_p(\cdot,\cdot)$ is a linear kernel and $\mathcal{M}_j$ is $\mathcal{M}_{Ie}^{(2)}$, let $\mathsf{m}_2$ correspond to value $yellow$, $k_p(\mathsf{m}_2,\mathsf{m}_2) = [0, \frac{2}{3}, \frac{1}{2}, \frac{1}{2}, 0]^{\top} \cdot [0, \frac{2}{3}, \frac{1}{2}, \frac{1}{2}, 0] = \frac{17}{18}$.

To reveal the heterogeneity within a coupling, UNTIE learns the weights of values in each kernel space. Specifically, it learns a set of transformation matrices $\{\mathbf{T}_1,\mathbf{T}_2,\cdots, \mathbf{T}_{n_k}\}$ to reconstruct the kernel spaces $\{\mathcal{K}_1^{\prime}, \cdots, \mathcal{K}_{n_k}^{\prime} \}$, in which the $p$-th kernel matrix $\mathbf{K}_p^{\prime}$ only contains the $p$-th kernel sensitive distribution that suits for the corresponding coupling. We call the reconstructed kernel spaces \textit{heterogeneous kernel spaces}.  $\mathbf{K}_p^{\prime}$ is defined as:
\begin{equation}\label{eq:trans}
    \mathbf{K}_p^{\prime} = \mathbf{T}_p\cdot \mathbf{K}_p.
\end{equation}
UNTIE regulates $\mathbf{T}_p$ as a diagonal matrix:
\begin{equation}
    \mathbf{T}_p = \left[
\begin{matrix}
    \alpha_{p1} & 0 & \cdots & 0\\
    0 & \alpha_{p2} & \cdots & 0\\
    \vdots & \vdots & \ddots & \vdots \\
    0 & 0 & \cdots & \alpha_{pn_v}\\
\end{matrix}
    \right].
\end{equation}
As a result, $\alpha_{pi}$ is the weight of the $i$-th value in the $p$-th kernel space, i.e., $[k_p(\mathsf{m_1},\mathsf{m_1}), k_p(\mathsf{m_1}, \mathsf{m_2}), \cdots, k_p(\mathsf{m_1}, \mathsf{m_{n_v^*})}]$.  The larger $\alpha_{pi}$ implies stronger coupling of the $i$-th value revealed by the coupling space corresponding to the $p$-th kernel space. 

To further capture the heterogeneity between couplings, UNTIE learns the contribution of each heterogeneous kernel space to a final representation. It first defines a similarity measure between objects in the heterogeneous kernel space and then learns the weight of each kernel space based on this similarity measure to reflect their contribution. Given a categorical data set, considering the $p$-th kernel matrix, let $\mathfrak{i}$ and $\mathfrak{j}$ represent the indices of values in the $p$-th kernel space corresponding to the $i$-th and $j$-th objects respectively, and using $\mathbf{K}_{p,\mathfrak{i}\cdot}^{\prime}$ (the $\mathfrak{i}$-th row in $\mathbf{K}_p^{\prime}$) to denote $\mathsf{o}_i$ in the $p$-th heterogeneous kernel space, the similarity $S_{p, ij}$ measured by the linear kernel of the $i$-th and $j$-th objects in this space is
\begin{equation}\label{eq:Sij}
S_{p, ij} = \mathbf{K}_{p,\mathfrak{i}\cdot}^{\prime\top}\mathbf{K}_{p,\mathfrak{j}\cdot}^{\prime}.
\end{equation}
By considering Eq. \eqref{eq:trans},  Eq. \eqref{eq:Sij} equals 
\begin{equation}
S_{p, ij} = \mathbf{K}_{p,\mathfrak{i}\cdot}^{\top}\mathbf{T}_p^{\top}\mathbf{T}_{p} \mathbf{K}_{p,\mathfrak{j}\cdot}.
\end{equation}
UNTIE defines the final similarity representation $S_{ij}$ between the $i$-th and $j$-th objects as a linear combination of base similarity measures from heterogeneous spaces to filter redundant information and integrate complementary information between couplings:
\begin{equation}\label{eq:sijc}
S_{ij} = \sum\limits_{p=1}^{n_k}\beta_p S_{p,ij},
\end{equation}
where $\beta_p \geq 0$ is the weight for the $p$-th base similarity.  Denoting a diagonal matrix $\bm{\omega}_p = \beta_p \mathbf{T}^{\top}\mathbf{T}$, Eq. \eqref{eq:sijc} is rewritten as:
\begin{equation}
S_{ij} = \sum\limits_{p=1}^{n_k}\mathbf{K}_{p,\mathfrak{i}\cdot}^{\top}\bm{\omega}_p\mathbf{K}_{p,\mathfrak{j}\cdot}.
\end{equation}
Accordingly, UNTIE simultaneously learns $\alpha$ and $\beta$ by learning $\bm{\omega}$, a heterogeneity parameter. The optimized $\bm{\omega}$ guides to integrate the heterogeneous couplings into the similarity representation $S_{ij}$.

\subsection{Kernel K-means-based Representation Learning}\label{subsec:unsupervised}
In an unsupervised way, UNTIE learns the heterogeneous couplings by wrapping $\alpha$ which reveals the heterogeneity within couplings and $\beta$ which reveals the heterogeneity between couplings into a wrapper kernel. It then further optimizes this kernel by solving a kernel k-means objective \cite{dhillon2004kernel}. 

K-means is a popular clustering algorithm that minimizes the distance between an object and its assigned cluster center, which was also used for information integration. Given a set of objects $O = \{\mathsf{o}_i \in \mathcal{R}^{n_a}|i=1,\cdots,n_o\}$, the k-means objective is formalized as:
\begin{equation}\label{eq:kmeans}
\begin{aligned}
& \underset{\mathbf{Z}\in\{0,1\}^{n_o \times n_c}}{\text{minimize}} 
& & \sum\limits_{i=1, c=1}^{n_o,n_c}z_{ic}\lVert \mathsf{o}_i - \bm{\mu}_c \rVert_{2}^{2}\\
& \text{subject to}
& & \sum\limits_{c=1}^{n_c} z_{ic} = 1,
\end{aligned}
\end{equation}
where $z_{ic}$ indicates whether $\mathsf{o}_i$ belongs to the $c$-th cluster, $\bm{\mu}_c = \frac{1}{n_{oc}}\sum_{i=1}^{n_o} z_{ic}\mathbf{x}_{i}$ is the centroid of the $c$-th cluster, and $n_{oc} = \sum_{i=1}^{n_o}z_{ic}$ refers to the size of the $c$-th cluster.

To address the issue that k-means cannot cluster data with a nonlinear boundary, the kernel k-means first uses a mapping function to map data to a higher dimensional space and then adopts k-means to cluster the mapped data. With a kernel function $k(\cdot)$, the kernel k-means is formalized as:
\begin{equation}\label{eq:kkmeans}
\begin{aligned}
& \underset{\mathbf{Z}\in\{0,1\}^{n_o \times n_c}}{\text{minimize}} 
& & \sum\limits_{i=1, c=1}^{n_o,n_{c}}z_{ic}\lVert k(\mathsf{o}_i) - \bm{\mu}_c \rVert_{2}^{2}\\
& \text{subject to}
& & \sum\limits_{c=1}^{n_c} z_{ic} = 1,
\end{aligned}
\end{equation}
where $\bm{\mu}_c = \frac{1}{n_{oc}}\sum_{i=1}^{n_o} z_{ic}k(\mathsf{o}_{i})$. 

Eq. \eqref{eq:kkmeans} is rewritten in the following form:
\begin{equation}\label{eq:kkmeans-matrix}
\begin{aligned}
& \underset{\mathbf{Z}\in\{0,1\}^{n_o \times n_c}}{\text{minimize}} 
& & \text{Tr}(\mathbf{K}) - \text{Tr}(\mathbf{L}^{\frac{1}{2}}\mathbf{Z}^{\top}\mathbf{K}\mathbf{Z}\mathbf{L}^{\frac{1}{2}})\\
& \text{subject to}
& & \mathbf{Z}\mathbf{1}_{n_c} = \mathbf{1}_{n_o},
\end{aligned}
\end{equation}
where $\text{Tr}(\cdot)$ calculates the trace of a matrix, $\mathbf{K}$ is a matrix with $k_{ij} = k(\mathsf{o}_i^{\top})k(\mathsf{o}_j)$, $\mathbf{L} = \text{diag}([n_{o1}^{-1}, n_{o2}^{-1}, \cdots, n_{on_c}^{-1}])$ and $\mathbf{1}_{\ell}\in \{1\}^{\ell}$ is a column vector with all elements being 1.

Directly solving Eq. \eqref{eq:kkmeans-matrix} is difficult since the values of $\mathbf{Z}$ are limited to either 0 or 1. Typically, Eq. \eqref{eq:kkmeans-matrix} is relaxed by letting $\mathbf{Z}$ take real values. Denoting $\mathbf{H} = \mathbf{ZL}^{\frac{1}{2}}$, the above problem is restated as \begin{equation}\label{eq:rkkmeans-matrix}
\begin{aligned}
& \underset{\mathbf{H}}{\text{minimize}} 
& & \text{Tr}(\mathbf{K}(\mathbf{I}_{n_o} - \mathbf{H}\mathbf{H}^{\top}))\\
& \text{subject to}
& & \mathbf{H}\in \mathcal{R}^{n_o \times n_c},\\
&
& & \mathbf{H}^{\top}\mathbf{H} = \mathbf{I}_{n_c},
\end{aligned}
\end{equation}
where $\mathbf{I}_{n_c}$ is an identity matrix with size $n_c \times n_c$. The optimal $\mathbf{H}$ for Eq. \eqref{eq:rkkmeans-matrix} can be obtained by taking the $n_c$ eigenvectors having large eigenvalues of $\mathbf{K}$ \cite{jegelka2009generalized}.

UNTIE integrates heterogeneous coupling learning into the kernel k-means seamlessly by wrapping $\alpha$ and $\beta$ to a wrapper kernel $s(\cdot,\cdot): O\times O \rightarrow \mathcal{R}$, which is defined below,
\begin{equation}\label{eq:wrapper_kernel}
    s(\mathsf{o}_i, \mathsf{o}_j) = S_{ij}.
\end{equation}
Accordingly, UNTIE constructs a kernel matrix $\mathbf{S}$ w.r.t. kernel $s(\cdot, \cdot)$ and categorical object set $O$:
\begin{equation}\label{eq:kernelmatrix}
\mathbf{S} = \left[
\begin{matrix}
s(\mathsf{o}_1, \mathsf{o}_1)& s(\mathsf{o}_1, \mathsf{o}_2)&\cdots& s(\mathsf{o}_1, \mathsf{o}_{n_o}) \\
s(\mathsf{o}_2, \mathsf{o}_1)& s(\mathsf{o}_2, \mathsf{o}_2)&\cdots& s(\mathsf{o}_2, \mathsf{o}_{n_o}) \\
\vdots & \vdots & \ddots & \vdots \\
s(\mathsf{o}_{n_o}, \mathsf{o}_1)& s(\mathsf{o}_{n_o}, \mathsf{o}_2)&\cdots& s(\mathsf{o}_{n_o}, \mathsf{o}_{n_o}) \\
\end{matrix}
\right].
\end{equation}
Since $s(\cdot,\cdot)$ is proved as a valid positive semi-definite kernel (see details in Section \ref{sec:psd-kernel}), $\mathbf{S}$ can replace $\mathbf{K}$ in Eq. \eqref{eq:rkkmeans-matrix}. In this way, the objective function of kernel k-means-based representation learning can be formalized as:
\begin{equation}\label{eq:obj}
\begin{aligned}
& \underset{\mathbf{H}, \bm{\omega}}{\text{minimize}} 
& & \text{Tr}(\mathbf{S}(\mathbf{I}_{n_o} - \mathbf{H}\mathbf{H}^{\top}))\\
& \text{subject to}
& & \mathbf{H}\in \mathcal{R}^{n_o \times n_c},\\
&
& & \mathbf{H}^{\top}\mathbf{H} = \mathbf{I}_{n_c},
\end{aligned}
\end{equation}
where $\bm{\omega}$ is a heterogeneity parameter to learn, and we obtain the similarity representation of categorical data as $\mathbf{S}$. The corresponding vector representation can be obtained by
\begin{equation}\label{eq:vetorrepresentation}
    \mathbf{x}_i = [\sqrt{\bm{\omega}_{1,11}}\mathbf{K}_{1,\mathfrak{i}1},\sqrt{\bm{\omega}_{1,22}}\mathbf{K}_{1,\mathfrak{i}2}, \cdots,\sqrt{\bm{\omega}_{n_k,n_v^*n_v^*}}\mathbf{K}_{n_k, \mathfrak{i}n_v^*}],
\end{equation} 
where $\bm{\omega}_{i,jj}$ refers to the value of the $(j,j)$-th entry in $\bm{\omega}_i$ and $n_v^*$ refers to the number of values in the attribute corresponding to the $n_k$-th kernel. The learned representation $\mathbf{x}_i$ is a numerical approximation of categorical data, which can be fed into vector-based learning methods. 


\subsection{The UNTIE Algorithm}

The UNTIE objective function in Eq. \eqref{eq:obj} can be solved by alternatively updating $\mathbf{H}$ and $\bm{\omega}$: (1) Optimizing $\mathbf{H}$ given $\bm{\omega}$: by fixing the parameter $\bm{\omega}$, $\mathbf{H}$ can be obtained by solving a kernel $k$-means clustering optimization problem shown in Eq. \eqref{eq:rkkmeans-matrix} by eigenvalue decomposition; (2) Optimizing $\bm{\omega}$ given $\mathbf{H}$: with $H$ fixed, the objective function of learning $\bm{\omega}$ is
\begin{equation}\label{eq:omega}
\begin{aligned}
& \underset{\bm{\omega}}{\text{minimize}} 
& & \text{Tr}(\mathbf{S}(\mathbf{I}_{n_o} - \mathbf{H}\mathbf{H}^{\top})),
\end{aligned}
\end{equation}
which can be optimized by linear programming. For large-scale data, Eq. \eqref{eq:omega} can be solved by the stochastic gradient descent (SGD) method, e.g., AdaGrad \cite{duchi2011adaptive} and Adam \cite{kingma2014adam}. We analyze the computational cost of UNTIE w.r.t different optimization methods in Section \ref{sec:complexity}.
Algorithm \ref{algorithm} explains the UNTIE working process.

\begin{algorithm}[!htpb]
        \caption{The UNTIE Algorithm for Unsupervised Categorical Representation}\label{algorithm}
        \small
    \begin{algorithmic}[1]
            \Require Categorical data set $C$, a set of kernel functions $K = \{k_1(\cdot,\cdot),\cdots,k_{n^*_k}(\cdot,\cdot)\}$, the number of clusters $n_c$, and convergence rate $\delta$.
            \Ensure Similarity representation $\mathbf{S}$, and vector representation $\mathbf{X}$.
            \State Mapping categorical data to coupling spaces according to Eqs. \eqref{eq:intra-couplings} and \eqref{eq:inter-coupling}.
            \State Mapping coupling spaces to multiple kernel spaces $\{\mathbf{K}_1, \cdots, \mathbf{K}_{n_k}\}$ by using $K$ according to Eq. \eqref{eq:kernelspace}. 
            \State Initializing the wrapper kernel matrix $\mathbf{S}$ by setting $\alpha$ and $\beta$ as $1$, and setting $l^{\prime} = +\infty$ and $\Delta = +\infty$.
            \For{$\Delta > \delta$}
            \State Calculating the $n_c$ eigenvectors that have the largest eigenvalues of $\mathbf{S}$. Constructing $\mathbf{H}$ by these eigenvectors.
            \State Optimizing  $\bm{\omega}$ by solving Eq. \eqref{eq:omega}.
            \State Calculating loss per $l = \text{Tr}(\mathbf{S}(\mathbf{I}_{n_o} - \mathbf{H}\mathbf{H}^{\top}))$.
            \State Calculating loss change per $\Delta = |l - l^{\prime}|$
            \State Setting $l^{\prime} = l$.
            \State $n_i = n_i + 1$.
            \EndFor
            \State Calculating the vector representation $\mathbf{X}$ per Eq. \eqref{eq:vetorrepresentation}.
            \State \Return{$\mathbf{S}$, $\mathbf{X}$}
    \end{algorithmic}
    \vspace{-3pt}
\end{algorithm}

\section{Theoretical Analysis}\label{sec:theory}
\subsection{The Fitness of Heterogeneity Hypotheses}\label{sec:hypothesis-proof}


To discuss the fitness of heterogeneity hypotheses, we first introduce the following theorem. 

\begin{theorem}\label{thm:factorize}
     The distribution $\Phi$ of a categorical data set can be described as a probability tensor $\bm{\Phi}$, where each entry corresponds to the joint probability of a set of categorical values from $n_a$ different attributes. $\bm{\Phi}$ can be decomposed as $\bm{\Phi} = \sum\limits_{h=1}^{k} \pi_h \Theta_h$,
	where $\Theta_h = \bm{\theta}_{h}^{(1)} \otimes \bm{\theta}_{h}^{(2)} \otimes \cdots \otimes \bm{\theta}_{h}^{(n_a)} $, $\pi_h$ is the weight of $\Theta_h$ for composing $\bm{\Phi}$, $\otimes$ refers to the outer product, $\bm{\theta}_h^{(j)}$ is a probability vector of categorical values in the $j$-th attribute with a size of $n_v^{(j)} \times 1$ for $h = 1, \cdots, k$ and $j = 1, \cdots, n_a$.
\end{theorem}
Theorem \ref{thm:factorize} can be proved by  Corollary 1 in \cite{dunson2009nonparametric}. The categorical data distribution $\Phi$ is a joint distribution of categorical values in each attribute. It is defined as $\Phi = \{\phi_{v^{(1)}v^{(2)}\cdots v^{(n_a)}} | v^{(j)} \in V^{(j)}\}$, where $\phi_{v^{(1)}v^{(2)}\cdots v^{(n_a)}}$ are the probabilities of values $v^{(1)},v^{(2)},\cdots,v^{(n_a)}$ co-occur. $\bm{\Phi}$ is a probability tensor, where each entry is an element in $\Phi$.
Theorem \ref{thm:factorize} indicates the following categorical data characteristics: 
\begin{itemize}
	\item \textit{A value may belong to multiple distributions}. For different $h$ values in Theorem \ref{thm:factorize}, a categorical value in $\mathsf{a}_j$ have different distributions $\bm{\theta}_h^{(j)}$. Therefore, if $h > 1$, the categorical value may have different distributions.
	\item \textit{A coupling may contribute differently to respective distributions}. The interactions under various distributions may differ. For distribution $\Theta_h$, the attributes are independent (indicated by $\Theta_h$ which equals the outer product of attribute distributions). In this case, inter-attribute couplings may not make contribution. On the contrary, for distribution $\Phi$, the attributes interact with each other, which is mainly reflected by inter-attribute couplings. 
	\item \textit{The overall distribution of a categorical value is a mixture of multiple distributions}. In Theorem \ref{thm:factorize}, the overall distribution of a categorical value equals the weighted sum of its multiple distributions. The parameter $\pi_h$ reflects the interactions between the distributions.
\end{itemize}
The heterogeneity hypotheses fit the above categorical data characteristics and provide a solid foundation for UNTIE to effectively capture the heterogeneity in couplings.

\subsection{The Positive Semi-definite Wrapper Kernel}\label{sec:psd-kernel}

As stated in Section \ref{subsec:unsupervised}, $s(\cdot, \cdot)$ has to be a positive semi-definite kernel to enable that $\mathbf{S}$ can be integrated into kernel k-means objective Eq. \eqref{eq:obj}. Before proving the above, we introduce a lemma of kernel properties.
\begin{lemma}\label{lemma:psd-op}
If $k_1(\mathsf{o}_i,\mathsf{o}_j)$ and $k_2(\mathsf{o}_i,\mathsf{o}_j)$ are positive semi-definite kernels in $O \times O$, and a constant $a > 0$, then the following $k(\mathsf{o}_i,\mathsf{o}_j)$ functions are positive semi-definite kernels:

(1) $k(\mathsf{o}_i,\mathsf{o}_j) = ak_1(\mathsf{o}_i,\mathsf{o}_j)$,

(2) $k(\mathsf{o}_i,\mathsf{o}_j) = k_1(\mathsf{o}_i,\mathsf{o}_j) + k_2(\mathsf{o}_i,\mathsf{o}_j)$.
\end{lemma}
This lemma can be found in Section 4.1 of \cite{steinwart2008support}.
\begin{theorem}\label{thm:psd}
The wrapper kernel $s(\mathsf{o}_i,\mathsf{o}_j) = S_{ij}$, which is defined in Eq. \eqref{eq:wrapper_kernel}, is a positive semi-definite kernel.
\end{theorem}
\begin{proof}
Given coupling spaces and multiple kernel functions, the $i$-th object $\mathsf{o}_i$ corresponds to a real value vector $\mathbf{K}_{p,\mathfrak{i}\cdot}^{\prime} \in \mathcal{R}^{n_o}$ in the $p$-th kernel space. Therefore, $s_p(\mathsf{o}_i, \mathsf{o}_j) = S_{p,ij} =\mathbf{K}_{p,\mathfrak{i}\cdot}^{\prime\top}\mathbf{K}_{p,\mathfrak{j}\cdot}^{\prime}$ is a well-known linear kernel, which is positive semi-definite. Treating $s_p(\mathsf{o}_i, \mathsf{o}_j)$ as $k_1(\mathsf{o}_i, \mathsf{o}_j)$ in Lemma \ref{lemma:psd-op}, since $\beta_p \geq 0$, consequently, $\beta_ps_p(\mathsf{o}_i,\mathsf{o}_j)$ is a positive semi-definite kernel according to Formula (1) of Lemma \ref{lemma:psd-op}.  Consequently, as the wrapper kernel $s(\mathsf{o}_i,\mathsf{o}_j) = S_{ij} = \sum\limits_{p=1}^{n_k}\beta_p S_{p,ij}$ is an accumulative summation of $\beta_ps_p(\mathsf{o}_i,\mathsf{o}_j)$,  $s(\mathsf{o}_i,\mathsf{o_j})$ is positive semi-definite by repeatedly adopting Formula (2) of Lemma \ref{lemma:psd-op} (treating $\sum\limits_{p=1}^{q}\beta_pS_{p,ij}$ as $k_1(\mathsf{o}_i, \mathsf{o}_j)$, and $\beta_{q+1}S_{q+1, ij}$ as $k_2(\mathsf{o}_i, \mathsf{o}_j)$ for $1 \leq q \leq n_k$).
\end{proof}

Theorem \ref{thm:psd} guarantees that the kernel matrix $\mathbf{S}$, which is constructed by kernel $s(\mathsf{o}_i, \mathsf{o}_j)$ as Eq. \eqref{eq:kernelmatrix}, can be incorporated into the kernel k-means objective. This is a fundamental property to support the effective unsupervised learning by UNTIE.

\subsection{The Separability of UNTIE-represented Data}\label{sec:normalized-cut}
The separability of a representation can be measured according to the overlap between object sets (e.g., clusters or classes) w.r.t. the representation. The UNTIE-represented data is with good separability since the resultant representation has the minimum normalized cut, which reflects the minimum overlap. Here, we first define the normalized cut and then prove that the objective of UNTIE is equivalent to learning a representation with the minimum normalized cut.

Given a graph consisting of categorical data $G = <O, \mathbf{A}>$, where $O$ is the set of categorical objects and $\mathbf{A}$ is a non-negative and symmetric affinity matrix that contains the connected strength or similarity between objects: 
\newtheorem{definition}{Definition}
\begin{definition} \label{def:cut}
The \textit{normalized cut} specifies the connection strength between two sets relative to the total connection strengths in a graph. Formally, the normalized cut between sets $O_1, O_2 \subseteq O$ is 
\begin{equation*}
normCut(O_1, O_2) = \frac{\sum_{i\in O_1, j\in O_2}\mathbf{A}(i,j)}{\sum_{i\in O_1, j\in O}\mathbf{A}(i,j)}.
\end{equation*}
\end{definition}  
Definition \ref{def:cut} shows that the normalized cut indicates the overlap between object sets. In other words, the \textit{minimum normalized cut} reflects the maximum separability between clusters, which is essential for most of learning tasks, e.g., clustering. 
\begin{theorem}\label{thm:cut}
The objective of UNTIE in Eq. \eqref{eq:obj} is equivalent to learning a representation with the minimum normalized cut.
\end{theorem}
\begin{proof}
The objective of minimizing the normalized cuts between clusters can be formalized as,
\begin{equation*}
\begin{aligned}
& \text{minimize}
& & \frac{1}{n_c}\sum\limits_{j=1}^{n_c} normCut(O_j, O \setminus O_j),
\end{aligned}
\end{equation*}
where $O_j$ refers to the set of objects in the $j$-th cluster, and $O \setminus O_j$ refers to the set of objects in the clusters instead of $j$-th cluster. This objective function can be converted to a trace maximization problem according to \cite{stella2003multiclass} as follows,
\begin{equation*}
\begin{aligned}
& \text{maximize}
& & \frac{1}{n_c}\text{Tr}(\mathbf{V}^{\top}\mathbf{A}\mathbf{V}),
\end{aligned}
\end{equation*}
where $\mathbf{V} = \mathbf{Z}(\mathbf{Z}^{\top}\mathbf{D}\mathbf{Z})^{-\frac{1}{2}}$, and $\mathbf{Z}\in\{0,1\}^{n_0 \times n_c}$ is an indicator matrix for the cluster, and $\mathbf{D}$ is the diagonal matrix, in which ($i$,$i$)-entry is the sum of the $i$-th row in $\mathbf{A}$. If further relaxing the matrix $\mathbf{Z}$ to a real value matrix and denoting $\mathbf{H} = \mathbf{D}^{\frac{1}{2}}\mathbf{V}$, the objective function can be converted to
\begin{equation*}
\begin{aligned}
& \text{maximize}
& & \frac{1}{n_c}\text{Tr}(\mathbf{H}^{\top}\mathbf{D}^{-\frac{1}{2}}\mathbf{A}\mathbf{D}^{-\frac{1}{2}}\mathbf{H})\\
& \text{subject to}
& & \mathbf{H} \in \mathcal{R}^{n_o \times n_c}\\
&
& &\mathbf{H}^{\top}\mathbf{H} = \mathbf{I}_{n_c}.
\end{aligned}
\end{equation*}
Let $\mathbf{D}^{-\frac{1}{2}}\mathbf{A}\mathbf{D}^{-\frac{1}{2}}$ be  $\mathbf{S}$ in Eq. \eqref{eq:obj}, by adding a low rank regularization $\text{Tr}(\mathbf{S})$, this objective function is equivalent to the UNTIE objective function Eq. \eqref{eq:obj}. Therefore, the objective of UNTIE in Eq. \eqref{eq:obj} is equivalent to learning a representation with the minimum normalized cut.
\end{proof}

\subsection{Convergence of the UNTIE Algorithm}\label{sec:convergency}
\begin{theorem}\label{thm:convergency}
	The UNTIE algorithm converges to a local minimal solution in a finite number of iterations.
\end{theorem}
\begin{proof}
	Let $y$ be the number of all possible partitions of a categorical data set $C$, and each partition can be represented by an indicator matrix $\mathbf{H}$. If two partitions are different, their indicator matrices are also different; otherwise, they are identical. We note that $y$ is finite, given $C$ and the number of cluster $n_c$. Therefore, there are a finite number of $\mathbf{H}$ on $C$. While applying UNTIE to cluster $C$, we obtain a series of $\mathbf{H}$, i.e., $\mathbf{H}_1, \mathbf{H}_2, \cdots, \mathbf{H}_{n_i}$, and a series of $\bm{\omega}$, i.e., $\bm{\omega}_1, \bm{\omega}_2, \cdots, \bm{\omega}_{n_i}$, along the iterations. Given a matrix $\mathbf{H}$ and a heterogeneity parameter $\bm{\omega}$, denote the loss of the UNTIE objective function Eq.\eqref{eq:obj} as $l_{\mathbf{H},\bm{\omega}}$. Since kernel $k$-means and linear programming for Eq. \eqref{eq:omega} converge to minimal solutions, $l_{\mathbf{H},\bm{\omega}}$ is strictly decreasing, i.e.,  $l_{\mathbf{H}_1,\bm{\omega}_1} \geq l_{\mathbf{H}_2,\bm{\omega}_2} \geq \cdots \geq l_{\mathbf{H}_{n_i},\bm{\omega}_{n_i}}$. We assume that the number of iterations $n_i$ is more than $y+1$. That indicates there are at least two same indicator matrices in the sequence, i.e., $\mathbf{H}_i = \mathbf{H}_j$, $1 \leq i \neq j \leq n_i$. For $\mathbf{H}_i$ and $\mathbf{H}_j$, we have the optimized heterogeneity parameter $\bm{\omega}_i$ and $\bm{\omega}_j$, respectively. It is clear that $\bm{\omega}_i = \bm{\omega}_j$ since $\mathbf{H}_i = \mathbf{H}_j$. Therefore, we obtain $l_{\mathbf{H}_i,\bm{\omega}_i} = l_{\mathbf{H}_j, \bm{\omega}_i} = l_{\mathbf{H}_j, \bm{\omega}_j}$, i.e., the value of objective function does not change. If the value of the objective function does not change, UNTIE stops, and $n_i$ is not more than $y+1$. UNTIE converges to a local minimal solution in a finite number of iterations.
\end{proof}

\subsection{Computational Efficiency}\label{sec:complexity}
The time complexity of UNTIE is determined by two parts, i.e., building coupling spaces and learning heterogeneities. 
In building coupling spaces, the time cost depends on what kind of couplings UNTIE captures. In this paper, UNTIE captures the intra- and inter-attribute couplings. The intra-attribute couplings measure the frequency of each value, corresponding to complexity $O(n_v)$. The inter-attribute couplings calculate the relationship between each value pair in each attribute pair, corresponding to the time cost $O(n_{mv}^2n_{a}^2)$, where $n_{mv}$ is the maximal number of values in attributes. Consequently, the entire time complexity of building the coupling spaces is $O(n_v + n_{mv}^2n_{a}^2)$.

In heterogeneity learning, its time complexity is determined by the time cost of calculating eigenvectors and $\bm{\omega}$ and the number of iterations. If involving all data in each iteration, it requires $O(n_o^3)$ to calculate eigenvectors and $O((n_o^2)^{3.5}n_{\bm{\omega}}^2)$ to solve linear optimization of calculating $\bm{\omega}$  \cite{zhu2018heterogeneous}. Denoting the number of iterations as $n_i$, the time complexity of heterogeneity learning is $O(n_in_o^3 + n_i(n_o^2)^{3.5}n_{\bm{\omega}}^2)$, where $n_{\bm{\omega}}$ refers to the number of elements in $\bm{\omega}$. If taking stochastic optimization, denoting the number of object pairs in each batch as $n_b$, the time complexity is only  $O(n_b^3n_i + n_bn_{\bm{\omega}}n_i)$. Since the batch size $n_b \ll n_o$ holds for large data sets, stochastic optimization is much more efficient if it can converge within a small number of iterations.
We thus recommend stochastic optimization to solve the UNTIE objective. Accordingly, the time complexity of UNTIE is $O(n_v + n_{mv}^2n_{a}^2+n_b^3n_i + n_bn_{\bm{\omega}}n_i)$.

The space complexity of UNTIE 
with stochastic optimization is $O(n_o^2n_{\bm{\omega}})$. For large categorical data, the space complexity is very high when full data optimization is used, which approaches $O(n_o^2)$. However, conducting stochastic optimization to obtain an approximate solution can largely reduce the space complexity since $n_b \ll n_o^2$. Therefore, we take stochastic optimization in UNTIE to tackle large data.

Overall, UNTIE has the time complexity $O(n_v + n_{mv}^2n_{a}^2+n_b^3n_i + n_bn_{\bm{\omega}}n_i)$ and space complexity $O(n_bn_{\bm{\omega}})$. This means UNTIE is scalable for large data. In addition, UNTIE can be further sped up by distributed and parallel computing in both building coupling spaces and learning heterogeneities, which will be explored in our future work. 

\section{Experiments}\label{sec:experiment}

\subsection{Experimental Settings}\label{sec:parameters}

\subsubsection{Data Sets}

25 data sets\footnote{More details of these data sets can be found on http://archive.ics.uci.edu/ml and https://www.sgi.com/tech/mlc/db. While we could not find much larger public categorical data, these small sets have rich data characteristics which challenge categorical representation.} are used to evaluate UNTIE. The data includes medical data: Hepatitis, Audiology, Spect, Mammographic, Wisconsin, Breastcancer, Primarytumor, and Dermatology; gene data: DNAPromoter, DNANominal, and Splice; social and census data: Housevotes, and Hayesroth; hierarchical decision-making data: Krvskp, Tictactoe, and Connect4; nature data: Zoo, Soybeanlarge, Flare, and Mushroom; business data: Crx; disaster prediction data: Titanic; and synthetic data with heterogeneous couplings: Mofn3710, Led24, and ThreeOf9. In Table \ref{tab:data}, they show strong diversity in terms of data factors: the number of objects ($n_o$), the number of attributes ($n_a$), the number of classes ($n_c$), the average number of nominal values in each attribute ($n_{av}$), and the maximum number of nominal values in all attributes ($n_{mv}$). Specifically, the number of objects ranges from 101 to 67,557, and the number of attributes ranges from 4 to 69. The data sets contain both binary and multiple classes with the maximum number of 24 classes. The average and maximal numbers of distinct attribute values range from 2 to 10 and from 2 to 15, respectively. 


 
\begin{table}[!htbp]
   \centering
  \footnotesize
   \caption{Characteristics of 25 Categorical Data Sets}
     \begin{tabular}{l|lllll}
     \toprule
     Dataset & $n_o$ & $n_a$ & $n_c$ & $n_{av}$ & $n_{mv}$ \\
     \midrule

     Zoo   & 101   & 16    & 7     & 2.25  & 6 \\

     DNAPromoter & 106   & 57    & 2     & 4     & 4 \\

     Hayesorth & 132   & 4     & 3     & 3.75  & 4 \\
     
     Hepatitis & 155 & 13 & 2 & 2.77 & 3 \\

     Audiology & 200   & 69    & 24    & 2.23 & 7 \\

	Housevotes & 232 & 16 & 2 & 2 & 2\\ 
	
     Spect & 267   & 22    & 2     & 2     & 2 \\
  
     Mofn3710 & 300   & 10    & 2     & 2     & 2 \\
  
    Soybeanlarge & 307 & 35 & 19 & 3.77 & 8 \\
    
    Primarytumor & 339 & 17 & 21 & 2.47 & 4 \\
    
    Dermatology & 366 & 33 & 6 & 3.91 & 4 \\

     ThreeOf9 & 512   & 9     & 2     & 2     & 2 \\
     
     Wisconsin & 683 & 9 & 2 & 9.89 & 10 \\

   
     Crx  & 690   & 9     & 2     & 5     & 15 \\
     
     Breastcancer & 699 & 9 & 2 & 10.00 & 11 \\
     
     Mammographic & 830   & 4     & 2     & 5     & 7 \\
     
     Tictactoe & 958 &9 &2 & 3 & 3\\
     
     Flare & 1,066 & 11 & 6 & 3.73 & 8\\
     
     Titanic & 2,201 & 3 & 4 & 2 & 2\\
     
     DNANominal & 3,186 & 60 & 3 & 4 & 4\\
     
     Splice & 3,190 & 60 & 3 & 4.78 & 6\\
     
     Krvskp & 3,196 & 36 & 2 & 2.03 & 3\\
     
     Led24 & 3,200 & 24 & 10 & 2 & 2\\
     
     Mushroom & 5,644 & 22 & 2 & 4.45 & 9\\
     
     
     
     Connect4 & 67,557 & 42 & 3 & 3 & 3\\
     
     \bottomrule
     \end{tabular}%
     \label{tab:data}
 \end{table}%

\subsubsection{Representation and Downstream Learning Baselines} 

We evaluate the UNTIE representation (UNTIE for short) against (1) the categorical distance measure Hamming distance (Hamming); (2) five state-of-the-art categorical data representation methods: CDE \cite{jianembedding}, COS \cite{wang2015coupled}, DILCA \cite{ienco2012context}, Ahmad \cite{ahmad2007method}, and Rough \cite{cao2012dissimilarity}; (3) two unsupervised deep representation methods BiGAN \cite{donahue2016adversarial} and VAE \cite{kingma2013auto} based on the wide-and-deep network \cite{cheng2016wide} (denoted as BiGAN\_WD and VAE\_WD, respectively).

The various learning tasks \textit{clustering} (two simple but popular clustering methods k-means and k-modes), \textit{classification} (K-Nearest  Neighbor-KNN,  Support  Vector Machine-SVM,  Random Forest-RF,  and  Logistic  Regression-LR), and \textit{object retrieval} are undertaken on the UNTIE representations to test their downstream task performance.

\subsubsection{Evaluation Perspectives} 
The following aspects of UNTIE performance are evaluated.
\begin{itemize}
    \item \textbf{Representation quality}: to reveal why UNTIE produces better representation results;
    \item \textbf{Downstream effectiveness}: to evaluate whether the UNTIE representation effectively improves various downstream learning tasks; 
    \item \textbf{Efficiency}: to reflect the cost sensitivity of UNTIE in representing data w.r.t. different data characteristics; 
    \item \textbf{Flexibility}: to verify whether the UNTIE representation fits and upgrades different learning methods.
    \item \textbf{Stability}: to test UNTIE's parameter sensitivity.
\end{itemize}
To avoid the impact of class imbalance, we evaluate the UNTIE-enabled clustering and classification performance by F-score ($\%$). The higher F-score indicates better learning performance.

\subsubsection{Implementation Settings} 
The default settings for implementing UNTIE are as follows. The kernels used in UNTIE are 11 Gaussian kernels with width from $2^{-5}$ to $2^5$ and three Polynomial kernels with order from $1$ to $3$. We use the stochastic optimization method Adam \cite{kingma2014adam} to solve the UNTIE objective function with the initial learning rate $10^{-3}$, the batch size $20$, and the max number of iterations $1,000$. For the parameters of the baseline methods, we take their recommended settings. 
UNTIE is implemented in Python 3.5 and Tensorflow r1.2, all experiments are conducted in a Windows 10 workstation with Intel i5-5300 U CPU@2.30GHz and 8GB memory.

\subsection{Evaluating the UNTIE Representation Quality}

\subsubsection{Reducing Heterogeneous Couplings Inconsistency}
Here, we evaluate how UNTIE effectively reduces the inconsistency during learning heterogeneous couplings. First, we quantitatively measure the inconsistency by the intra- and inter-coupling heterogeneity indicators. Then, we analyze the relations between UNTIE-enabled clustering results and the inconsistency indicated by these indicators.

The intra-coupling heterogeneity indicator ($I_{intra}$) measures the degree of heterogeneity in value distributions. It assigns a higher heterogeneity degree to couplings if each value in the couplings has more significant diverse distributions. Intuitively, if a value has multiple distributions, its representations in different distributions may be inconsistent with each other. The larger difference its distributions have (i.e., higher heterogeneity degree), the stronger inconsistency may exist. In our experiment, $I_{intra}$ compares the difference between distributions of a value per each ground-truth cluster, which is formalized as follows:
\begin{equation}
I_{intra} =\frac{ \sum\limits_{j=1}^{n_a}\sum\limits_{i=1}^{n_v^{(j)}}\text{NM}_{n_c}\left(\sqrt{\sum\limits_{k=1}^{n_{c}}\left(\frac{|g^{(j)}(\mathsf{v_i^{(j)}})\cap g^{(c)}(\mathsf{c}_k)|}{|g^{(j)}(\mathsf{v_i^{(j)}})|}\right)^2}\right)}{n_a},
\end{equation}
where $\mathsf{c}_k$ refers to the $k$-th cluster, $n_{c}$ is the number of ground-truth clusters, function $g(\cdot)$ has the same definition as in Eq. \eqref{eq:intra}, and function $\text{NM}_{n_c}$ normalizes a value to $[0,1]$ w.r.t. $n_c$ clusters, which is defined as
\begin{equation}
\text{NM}_{n_c}(x) = 1 - \frac{1 - x}{1 - \sqrt{\sum\limits_{k=1}^{n_{c}}\left(\frac{1}{n_{c}}\right)^2}}. 
\end{equation} 
$I_{intra}$ reflects the inconsistency within couplings caused by heterogeneous data distributions. $I_{intra}$ is large if each value has diverse distributions in each cluster, otherwise small. A larger $I_{intra}$ indicates a stronger inconsistency within a coupling. In extreme cases, $I_{intra}$ is $1$ when each value only appears in a ground-truth cluster, and $I_{intra}$ is $0$ when every value has the same distribution in every ground-truth cluster. 

To show the distribution of intra-coupling inconsistency, we illustrate the probability density of $I_{intra}$ on 25 testing data sets in Fig. \ref{fig:intra_heterogeneity_distribution} per kernel density estimation (KDE)\footnote{Here we treat the intra-coupling inconsistency of the 25 testing data sets as the random variable to calculate its KDE.}. The KDE smoothly estimates the probability density of $I_{intra}$ within its range. A larger density indicates a higher probability of an inconsistency degree. As shown in Fig. \ref{fig:intra_heterogeneity_distribution}, most of data sets have strong inconsistency within couplings, which indicates the necessity of eliminating inconsistency in learning heterogeneous couplings. 

\begin{figure}[!hbtp] \centering
	\includegraphics[width=0.75\columnwidth]{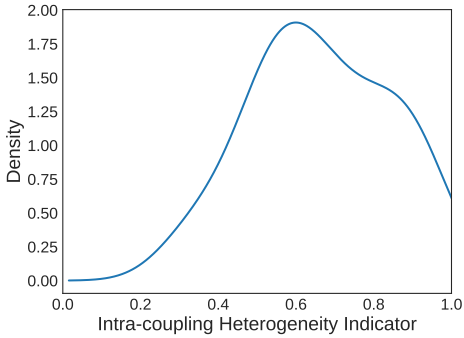}
	\caption{The Probability Density of Intra-coupling Heterogeneity Indicator per Kernel Density Estimation.}
	\label{fig:intra_heterogeneity_distribution}
\end{figure}

The inter-heterogeneity indicator ($I_{inter}$) represents the degree of difference between couplings. Intuitively, the increase of difference between couplings (i.e., higher heterogeneity degree) may increase the inconsistency between categorical data representations. In our experiment, $I_{inter}$ calculates the distance between coupling matrices to reflect the inter-coupling inconsistency. Here, a coupling matrix is a similarity matrix of the values in an attribute, which is defined by a coupling learning method. For UNTIE, a coupling matrix is calculated by the Euclidean distance w.r.t. a coupling vector representation of values (e.g., the intra-attribute coupling value representation as Eq. \eqref{eq:intra} or the inter-attribute coupling value representation as Eq. \eqref{eq:inter}). For CDE, a coupling matrix is calculated by the Euclidean distance w.r.t. a value clustering representation. For COS and DILCA, the coupling matrices are their value similarity matrices before weighted integration. Denoting the $k$-th coupling matrix of the $z$-th attribute as $\mathbf{C}^{(z,k)}$, the inter-heterogeneity indicator is defined as
\begin{equation}
I_{inter} = \sqrt{\frac{\sum\limits_{k=1}^{n_m^{(z)}}\sum\limits_{l=1}^{n_m^{(z)}}\frac{\sum\limits_{i=1}^{n_v^{(z)}}\sum\limits_{j=1}^{n_v^{(z)}}\left(\mathbf{C}^{(z,k)}_{ij} - \mathbf{C}^{(z,l)}_{ij}\right)^2}{n_v^{(z)2}}}{n_m^{(z)2}}},
\end{equation} 
$n_m^{(z)}$ is the number of couplings for the $z$-th attribute, and $\mathbf{C}_{ij}^{z,k}$ is the $(i,j)$-th entry of $\mathbf{C}_{ij}$. $I_{inter}$ differs for different coupling learning methods. A larger $I_{inter}$ indicates a stronger inconsistency between couplings. $I_{inter}$ is large if each coupling matrix is largely different from the others, and otherwise small. In an extreme case, $I_{inter}$ is $0$ when all coupling matrices are the same. 

The probability density of $I_{inter}$ on 25 testing data sets is shown in Fig. \ref{fig:inter_heterogeneity_distribution} per KDE, which shows UNTIE and CDE involve a higher degree of inter-coupling heterogeneity compared to DILCA and COS. On one hand, the higher degree of inter-coupling heterogeneity provides richer information for representation. On the other hand, the higher degree of inter-coupling heterogeneity may contain inconsistency representation with higher probability.

\begin{figure}[!hbtp] \centering
	\includegraphics[width=0.75\columnwidth]{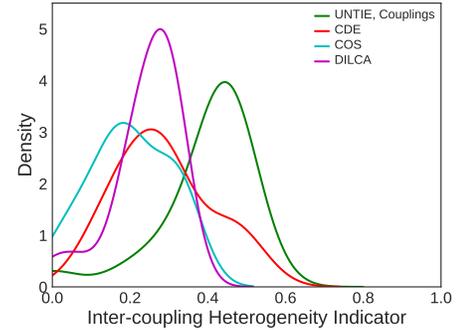}
	\caption{The Probability Density of Inter-coupling Heterogeneity Indicator per Kernel Density Estimation.}
	\label{fig:inter_heterogeneity_distribution}
\end{figure}

To evaluate whether UNTIE can effectively resolve the coupling inconsistency problem, we analyze the UNTIE-enabled clustering performance at different inconsistency levels indicated by the intra-/inter-coupling heterogeneity indicators. Considering the imbalanced distributions of these indicators as shown in Fig. \ref{fig:intra_heterogeneity_distribution} and Fig. \ref{fig:inter_heterogeneity_distribution}, we set  5 inconsistency levels, each of them contains the same number of data sets, according to the intra-/inter-coupling heterogeneity indicators. For example, the first level contains $20\%$ data sets with the smallest values of $I_{intra}$/$I_{inter}$, while the fifth level contains $20\%$ data sets with the largest values of $I_{intra}$/$I_{inter}$. We calculate the performance of a heterogeneous coupling learning method at an inconsistency level by averaging its enabled clustering rank in Table \ref{tab:clustering} on the data sets with the inconsistency level. We illustrate the relations between the UNTIE-enabled clustering performance and the inconsistency level in Fig. \ref{fig:intra_heterogeneity} and Fig. \ref{fig:inter_heterogeneity} per $I_{intra}$ and $I_{inter}$, respectively. 

As shown in Fig. \ref{fig:intra_heterogeneity}, UNTIE significantly outperforms its competitors on data sets at the inconsistency levels 2 to 5 in terms of $I_{intra}$, which have strong intra-coupling heterogeneity as shown in Fig. \ref{fig:intra_heterogeneity_distribution}. While other methods ignore the intra-coupling heterogeneity, UNTIE captures it by learning value weights in kernel spaces w.r.t. Eq. \eqref{eq:trans}. As a result, UNTIE reduces the inconsistency and achieves better performance. In contrast, on the data sets with the inconsistency level 1, UNTIE dose not show superiority. This is because these data sets do not have much coupling inconsistency, which further demonstrates that UNTIE gains better performance by reducing the inconsistency between heterogeneous couplings.

\begin{figure}[!hbtp] \centering
	\includegraphics[width=0.75\columnwidth]{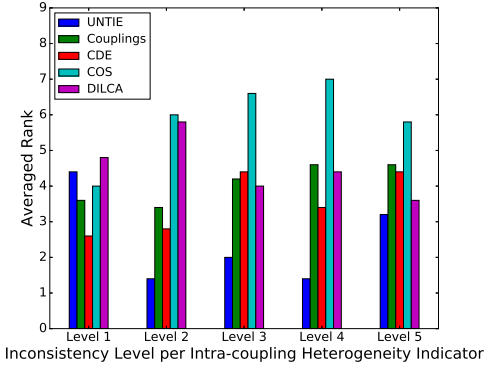}
	\caption{The UNTIE-enabled Clustering Performance on Data Sets with Different Inconsistency Levels per the Intra-heterogeneity Indicator.}
	\label{fig:intra_heterogeneity}
\end{figure}

In Fig. \ref{fig:inter_heterogeneity}, the UNTIE-enabled clustering performs much better than that enabled by existing heterogeneous coupling learning methods, especially on the data with higher inconsistency levels. Since the couplings represented by UNTIE may also incur inconsistency as shown in Fig. \ref{fig:inter_heterogeneity_distribution}, UNTIE well reduces the inconsistency to guarantee a good representation. This reduction is mainly contributed by the multiple kernels used in UNTIE, which enables to capture the fitness of couplings for different distributions. In Fig. \ref{fig:inter_heterogeneity}, UNTIE shows similar performance by simply combining multiple couplings on data sets with the inconsistency levels 2 and 3. This phenomenon indicates the heterogeneous couplings learned in our method capture much richer data information compared to other methods and enable better performance when the inconsistency between these couplings is not so substantial. However, with the inconsistency increase, simply combining these couplings may worsen the results, as shown on the data sets with the inconsistency levels 4 and 5 in Fig. \ref{fig:inter_heterogeneity}. 
  
\begin{figure}[!hbtp] \centering
	\includegraphics[width=0.75\columnwidth]{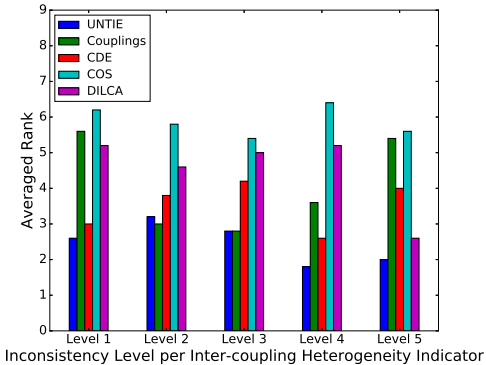}
	\caption{The UNTIE-enabled Clustering Performance on Data Sets with Different Inconsistency Levels per the Inter-heterogeneity Indicator.}
	\label{fig:inter_heterogeneity}
\end{figure}

\subsubsection{Goodness of the UNTIE-enabled Metric}

\begin{figure*}[htbp] \centering
\subfigure[$(\epsilon,\gamma)$-curve on Mofn3710.] { \label{fig:e}
\includegraphics[width=0.45\columnwidth]{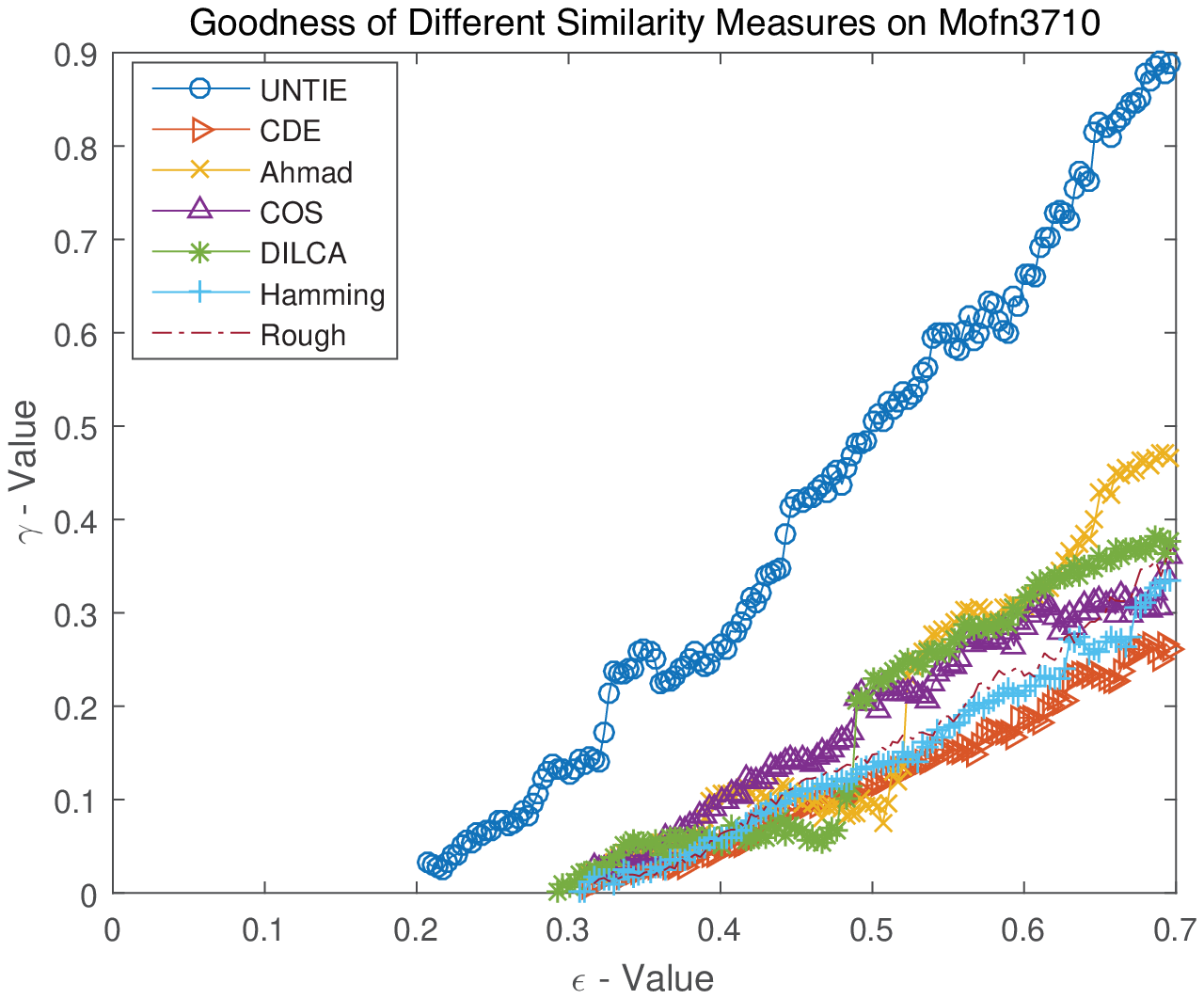}
}
\subfigure[$(\epsilon,\gamma)$-curve on Dermatology.] { \label{fig:f}
\includegraphics[width=0.45\columnwidth]{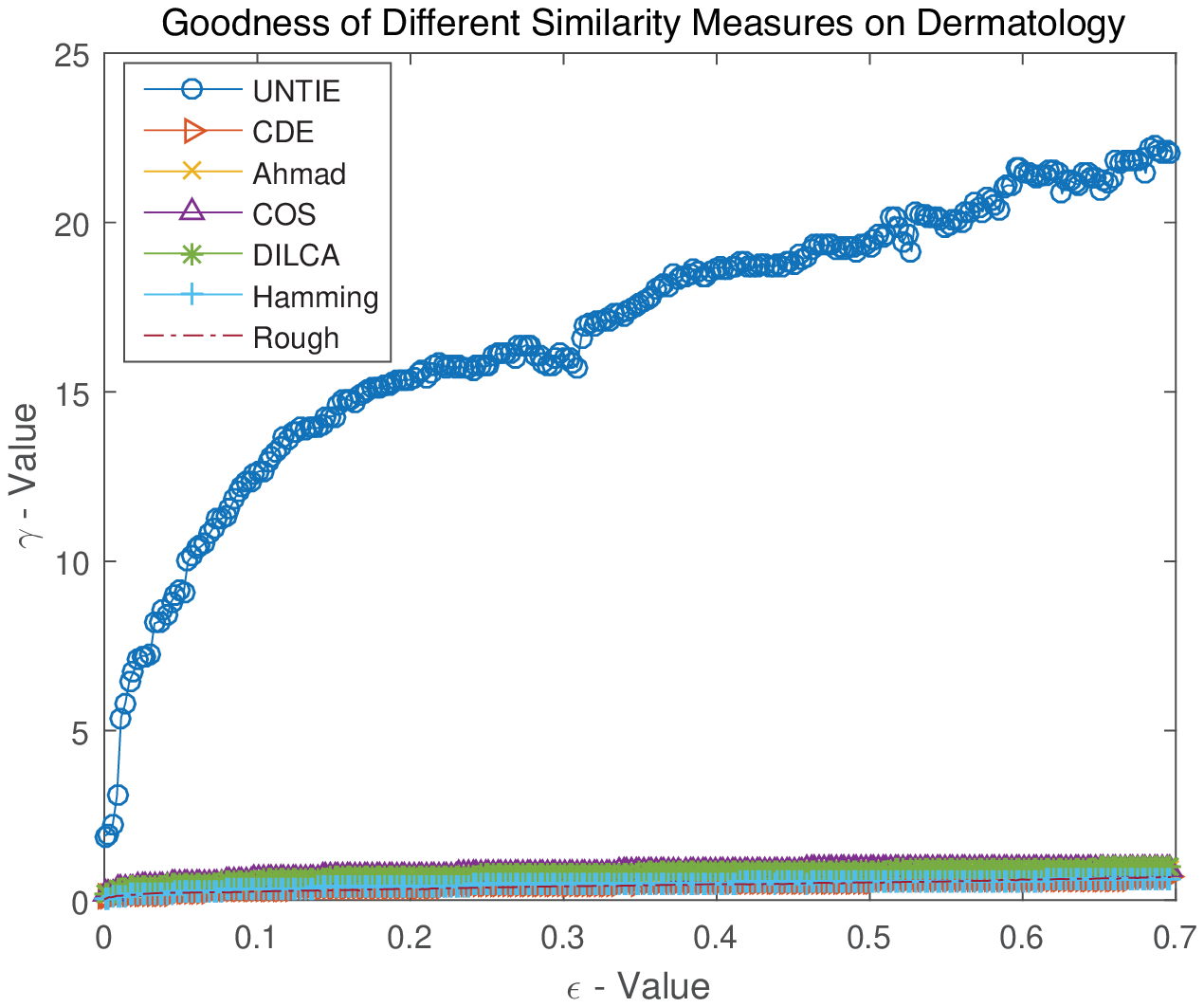}
}
\subfigure[$(\epsilon,\gamma)$-curve on Crx.] { \label{fig:b}
\includegraphics[width=0.45\columnwidth]{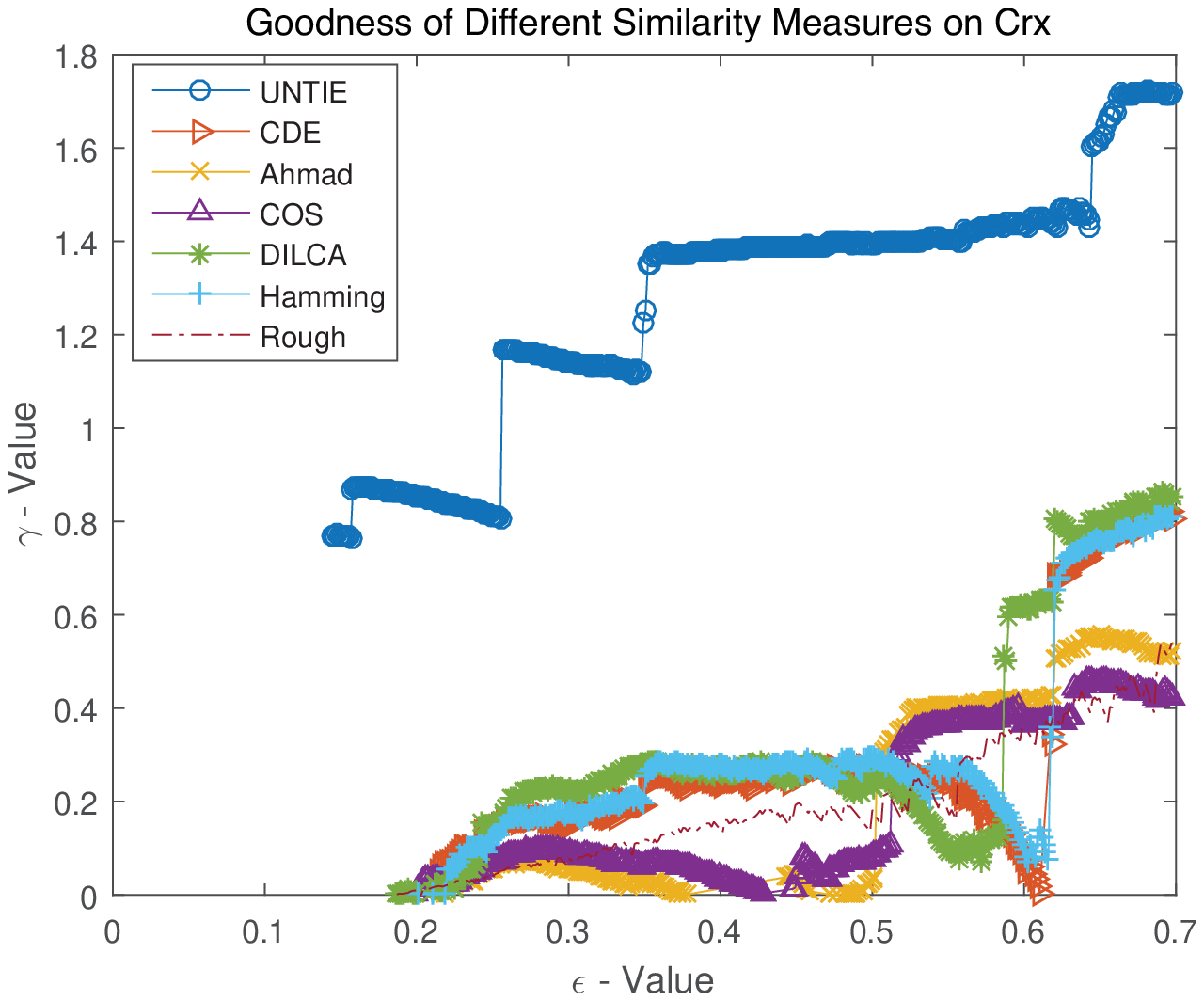}
}
\subfigure[$(\epsilon,\gamma)$-curve on Breastcancer.] { \label{fig:a}
\includegraphics[width=0.45\columnwidth]{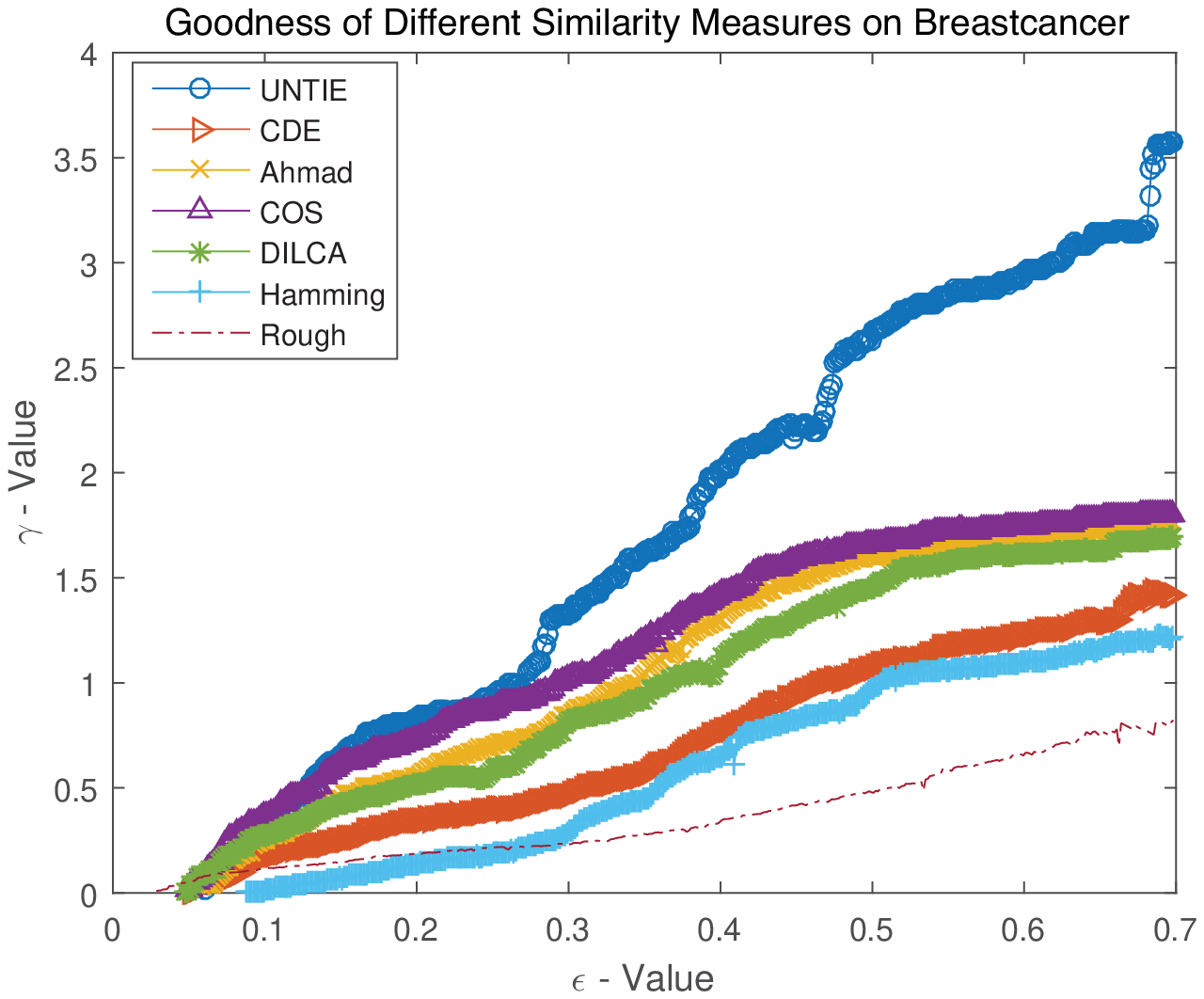}
}
\caption{The $(\epsilon,\gamma)$-curves of Different Transformed Similarity Measures: A better metric yields a better result.}
\label{fig:goodnessCurve}
\end{figure*}

Per Definition 4 in \cite{balcan2008theory}, a similarity function $Q$ (e.g., the similarity function $s(\cdot,\cdot)$ in Eq. \eqref{eq:wrapper_kernel}) is strongly $(\epsilon,\gamma)$-good  for a learning problem $P$ if at least a $1 - \epsilon$ probability mass of objects $\mathsf{o}$ satisfies:
\begin{equation*}
\begin{aligned}
    &E_{\mathsf{o},\mathsf{o}'\sim \Phi}[Q(\mathsf{o},\mathsf{o}')|c(\mathsf{o}) = c(\mathsf{o}')] \\
    &\geqslant E_{\mathsf{o},\mathsf{o}'\sim \Phi}[Q(\mathsf{o},\mathsf{o}')|c(\mathsf{o}) \neq c(\mathsf{o}')] + \gamma,
\end{aligned}
\end{equation*}
where $E$ refers to the expected value on distribution $\Phi$ in the learning problem $P$, and $c(\mathsf{o})$ refers to the class of $\mathsf{o}$. In this criterion, $\epsilon$ indicates the proportion of objects whose averaged intra-class similarity is not $\gamma$ degrees larger than their averaged inter-class similarity value. With the same $\gamma$, the smaller $\epsilon$ reflects the better similarity function for the learning problem $P$. Here, $\gamma$ indicates to what extent the intra-class similarity is larger than the inter-class similarity on the $1-\epsilon$ proportion of data which is best separated. When $\epsilon$ is fixed, the larger $\gamma$ reflects the better similarity function. The intuition of this criterion is that a good similarity measure should effectively differ data in the same class from those in other classes. More importantly, a $(\epsilon,\gamma)$-good similarity can induce a classifier with a bounded error (see more details in Theorem 1 of \cite{balcan2008theory}). 

In this experiment, we compare the UNTIE's similarity function $s(\cdot,\cdot)$ (defined in Eq. \eqref{eq:wrapper_kernel}) with the similarity functions in the other similarity-based representation methods per the $(\epsilon,\gamma)$-good criterion. For the CDE-learned vector representation, we reverse the Euclidean distance to measure the similarity between objects. Since different $\epsilon$ values may correspond to different $\gamma$ values, we draw the $(\epsilon,\gamma)$-curves to demonstrate the quality of the learned metric and the compared methods. With the same $\epsilon$, the better metric incurs a greater $\gamma$. In other words, the better metric yields a higher curve in the $(\epsilon,\gamma)$-curve. In this experiment, we draw the $(\epsilon,\gamma)$-curves on four data sets, i.e., 
Mofn3710, Dermatology, Crx and Breastcancer. The results are shown in Fig. \ref{fig:goodnessCurve}. It should be noted that, we only focus on $\epsilon$ that can guarantee a non-negative margin, i.e., $\gamma \geqslant 0$. Therefore, Fig. \ref{fig:goodnessCurve} only displays a part of the $(\epsilon,\gamma)$-curve, in which $\gamma \geqslant 0$.  

The results illustrate that UNTIE is better than its competitors in terms of the $(\epsilon,\gamma)$-good criterion. The results also reveal the insight behind the clustering performance in Table \ref{tab:clustering}. For data Dermatology and Crx, the UNTIE-enabled clustering has much higher F-score than others since UNTIE yields larger margins between different classes, which is reflected by the $(\epsilon,\gamma)$-good in Figs. \ref{fig:f} and \ref{fig:b}. For Mofn3710, all methods obtain low F-score, and the UNTIE-enabled clustering achieves the same result as CDE which is only slightly better than the other competitors. The reason is shown in Fig. \ref{fig:e}, where nearly $20\%$ of the Mofn3710 data cannot be well separated by UNTIE ($\gamma$ is 0 when $\epsilon$ is smaller than 0.2) while nearly $30\%$ of that cannot be well separated by its competitors. 
For other data sets, e.g., Breastcancer, all methods achieve good results since the $(\epsilon, \gamma)$-good criterion indicates that these methods can well separate the data sets.

\begin{figure*}[!hbtp] \centering
\subfigure[UNTIE-represented Distributions.] { 
\includegraphics[width=0.55\columnwidth]{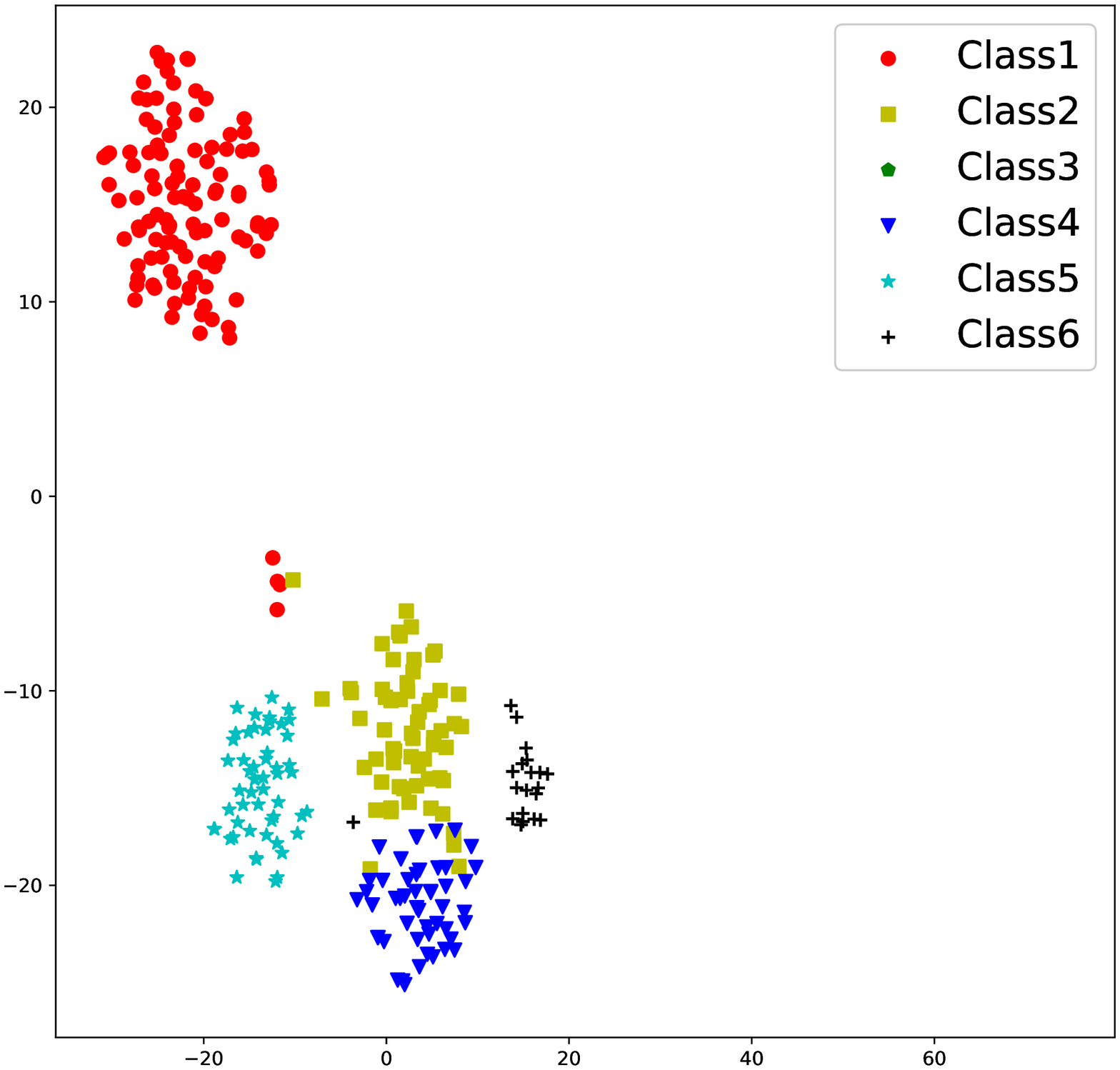}
}
\subfigure[CDE-represented Distributions.] { 
\includegraphics[width=0.55\columnwidth]{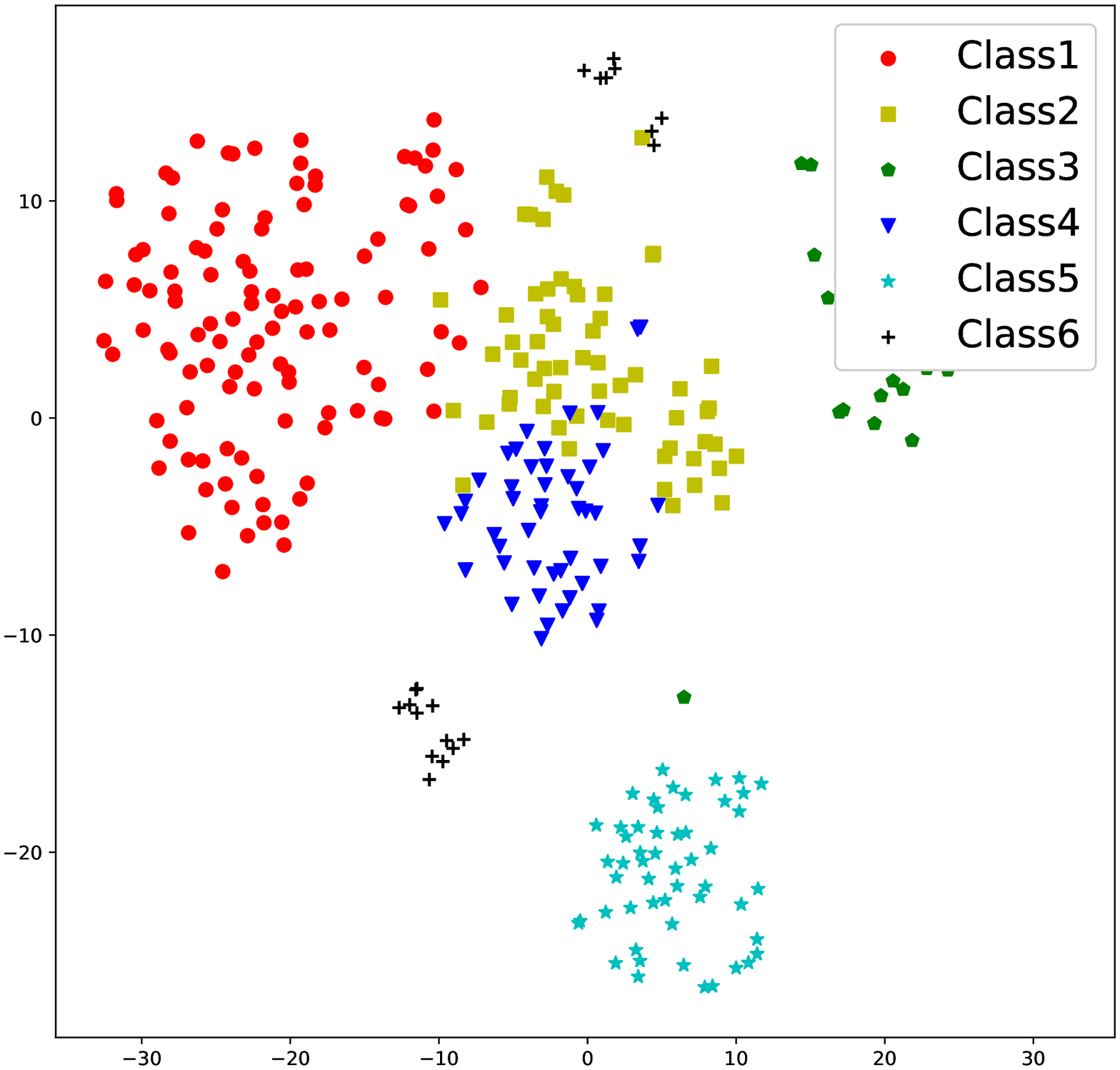}
}
\subfigure[BiGAN\_WD-represented Distributions.] { 
\includegraphics[width=0.55\columnwidth]{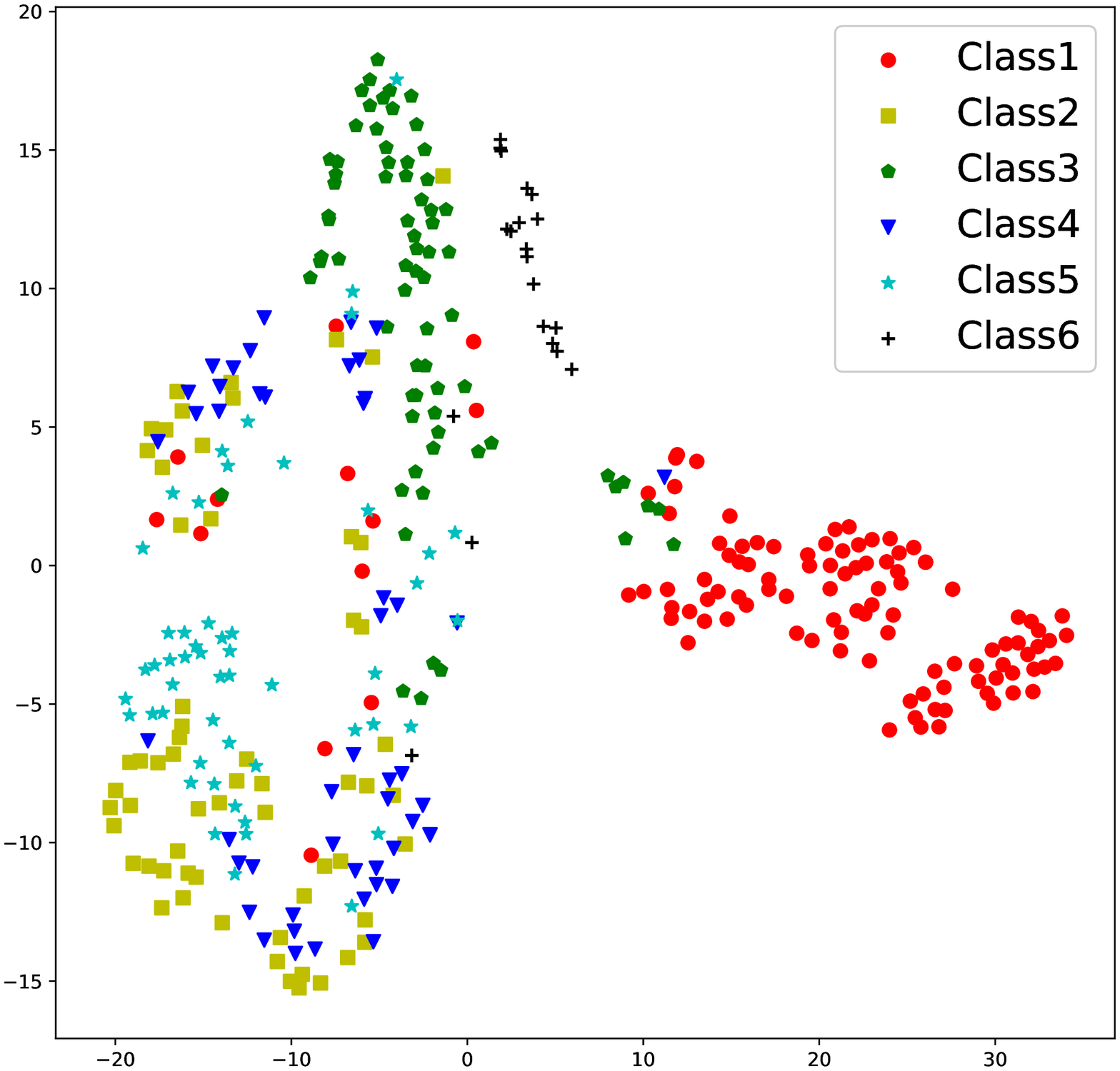}
}
\caption{The Visualization of Different Representation Methods on Dermatology. The UNTIE-represented data shows clearer boundaries between different clusters. The plotted two-dimensional embedding is converted from high-dimensional representation by t-SNE. Different symbols refer to different data clusters per the ground truth.}
\label{fig:visualize}
\end{figure*}

\subsubsection{Visualization of UNTIE Representations}

We visualize the separability of different representations. The UNTIE, CDE and BiGAN\_WD-represented data is converted from high-dimensional representation to two-dimensional embedding by the t-Distributed Stochastic Neighbor Embedding (t-SNE) \cite{maaten2008visualizing}. The basic idea of t-SNE is to minimize the  Kullback-Leibler divergence between the joint probabilities of the low-dimensional embedding and the high-dimensional representation. 
Fig. \ref{fig:visualize} shows the visualization of these different representation methods on Dermatology. The UNTIE-represented data has a more compact distribution and leads to clearer boundaries between different clusters, compared to that from CDE and BiGAN\_WD-represented data. It qualitatively demonstrates that the UNTIE representation is more separable and suitable for downstream tasks such as clustering and classification. This is because UNTIE learns the representation by optimizing the objective function Eq. \eqref{eq:obj}, i.e., by minimizing the distance between objects within a cluster and maximizing the distance between objects in different clusters.

\begin{figure}[!hbtp] \centering
\subfigure[Training Loss on Hepatitis.] { 
\includegraphics[width=0.28\columnwidth]{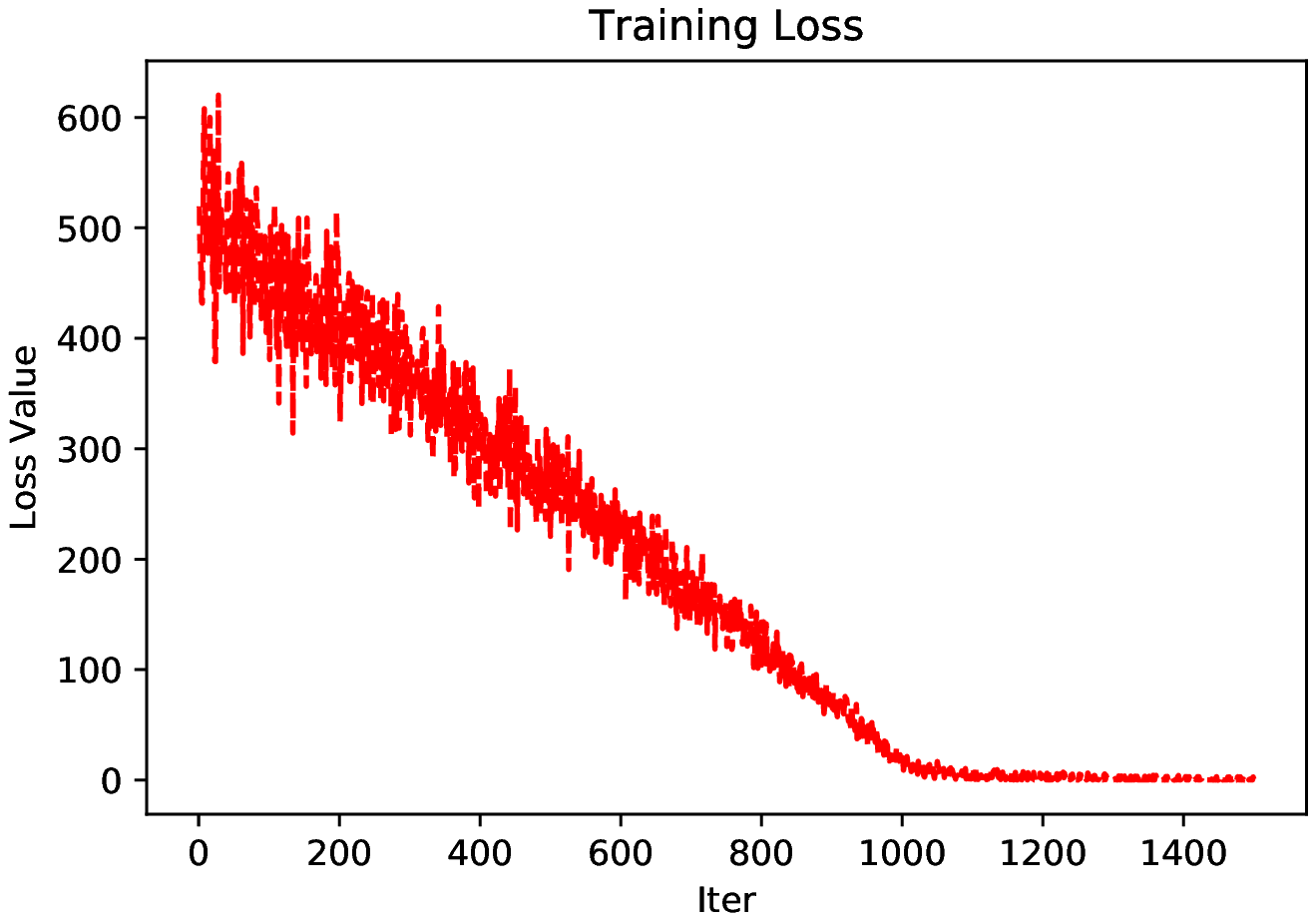}
}
\subfigure[Training Loss on Housevotes.] { 
\includegraphics[width=0.28\columnwidth]{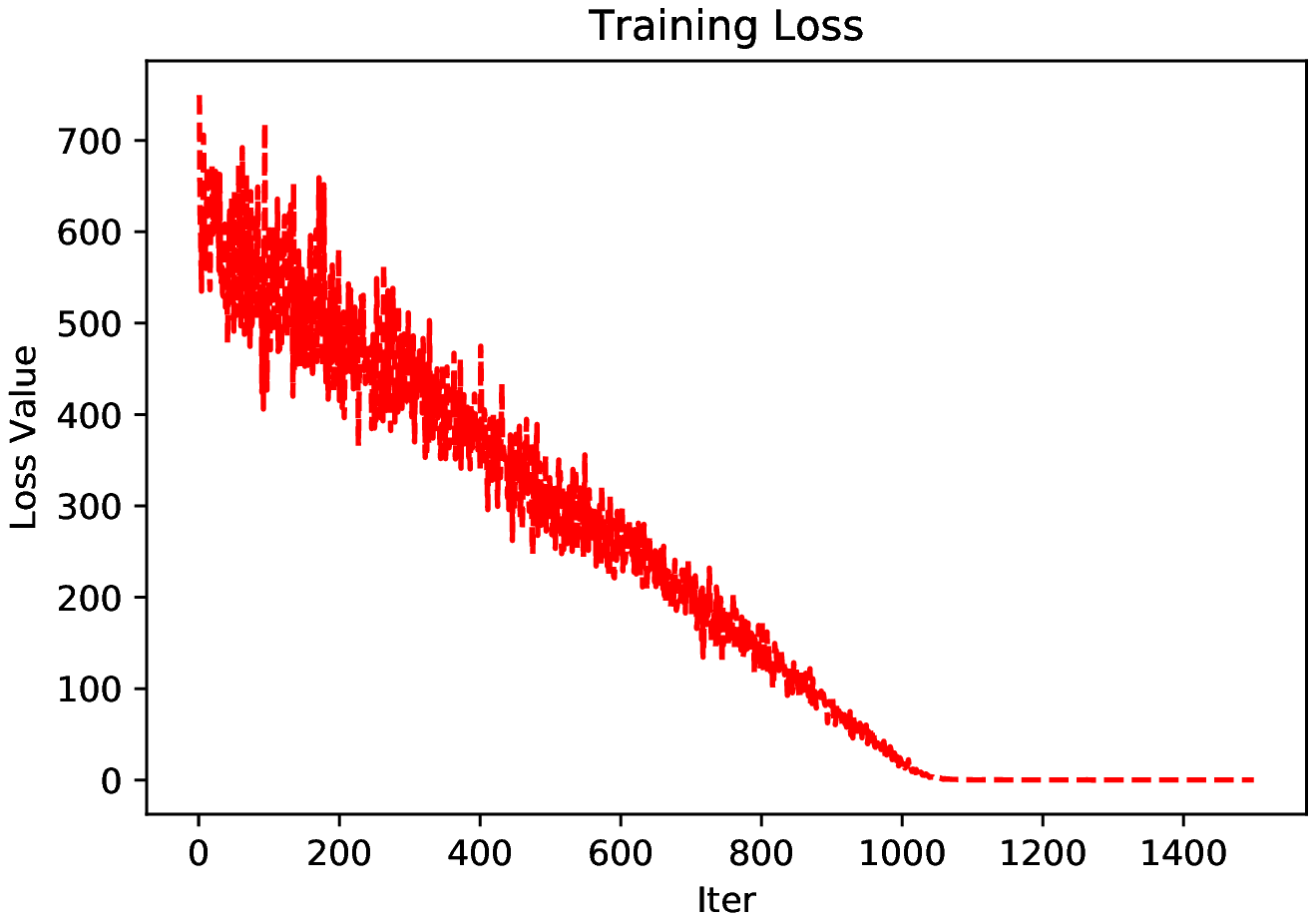}
}
\subfigure[Training Loss on Mofn3710.] { 
\includegraphics[width=0.28\columnwidth]{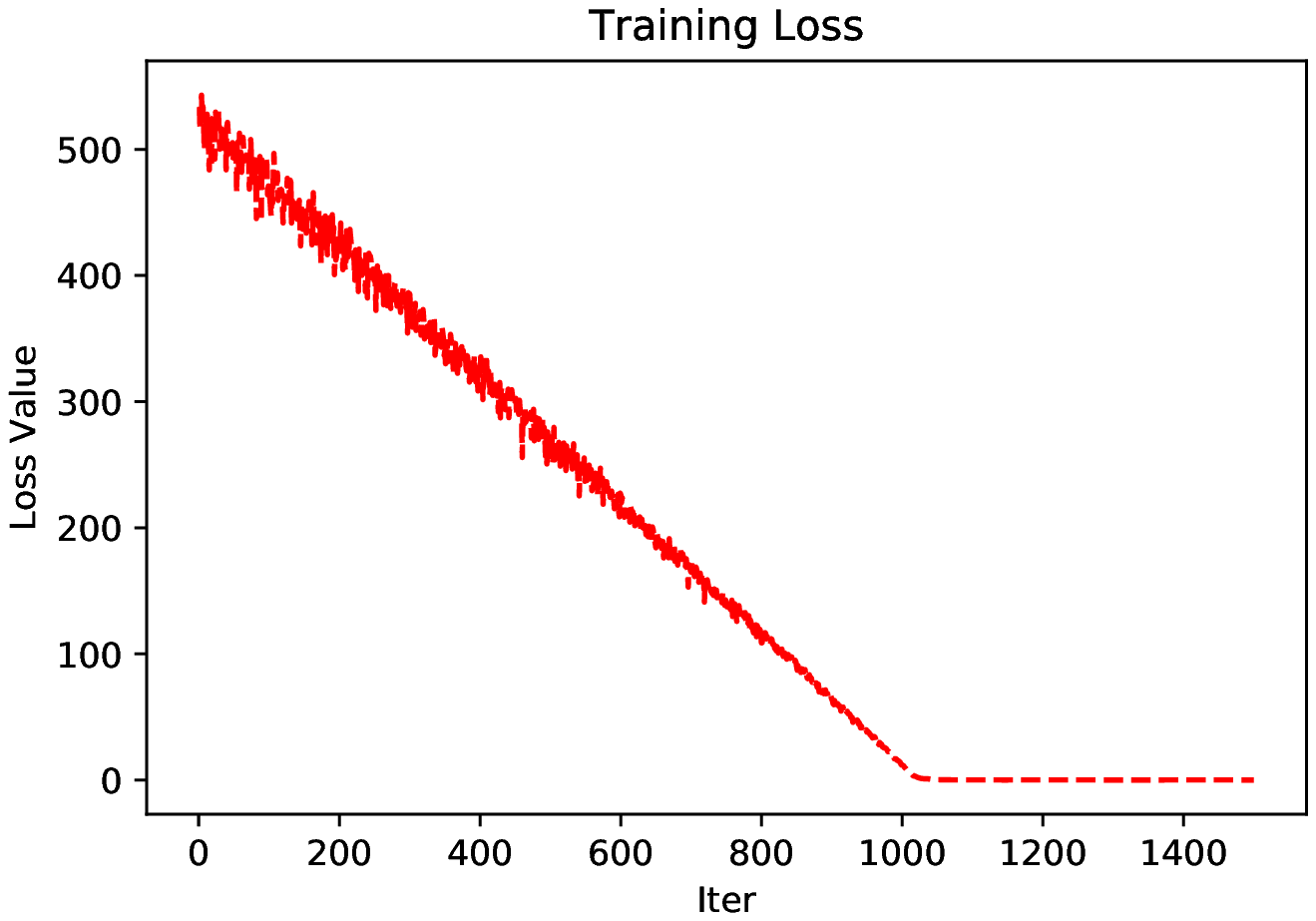}
}
\subfigure[Training Loss on Dermatology.] { 
\includegraphics[width=0.28\columnwidth]{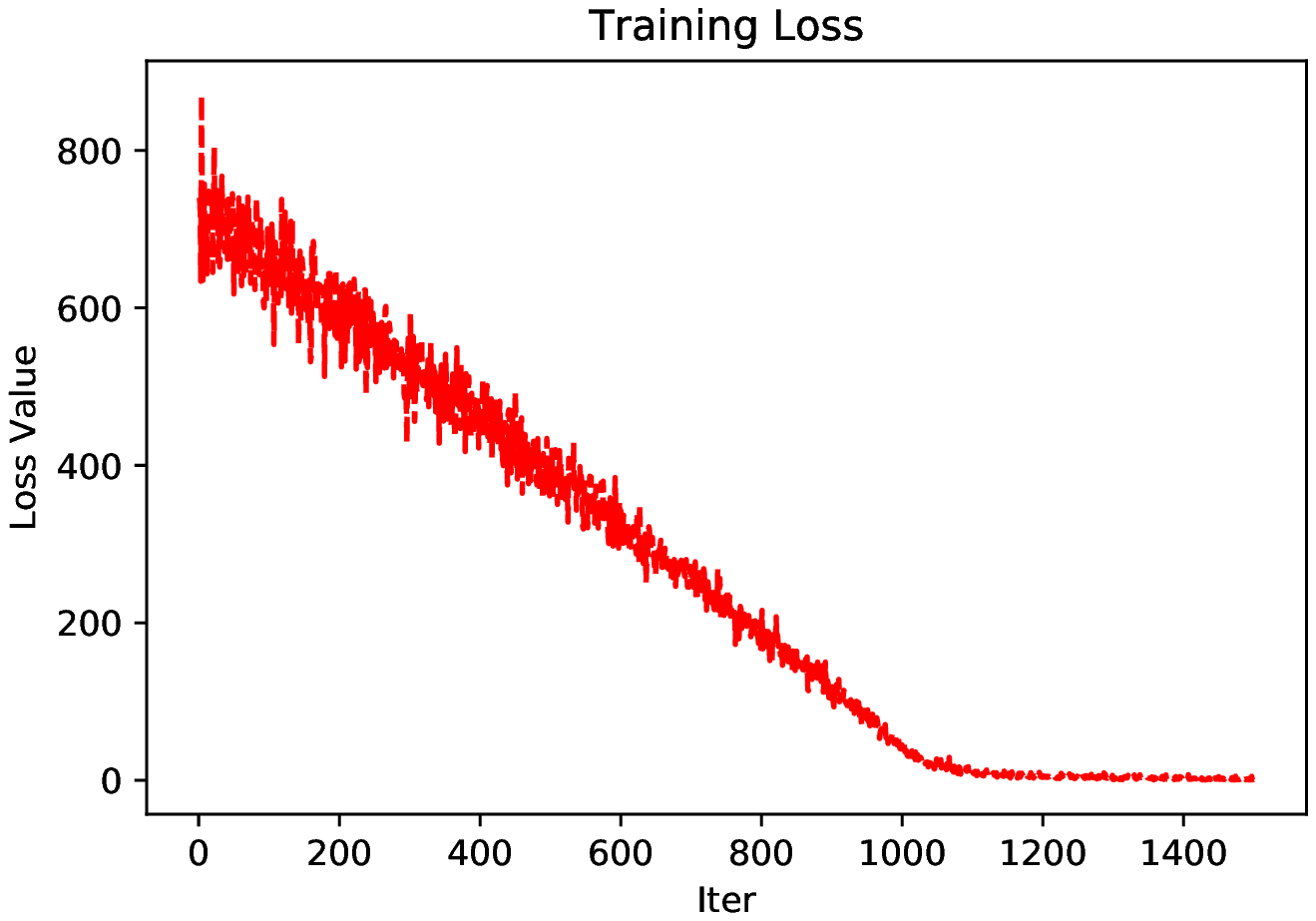}
}
\subfigure[Training Loss on Crx.] {
\includegraphics[width=0.28\columnwidth]{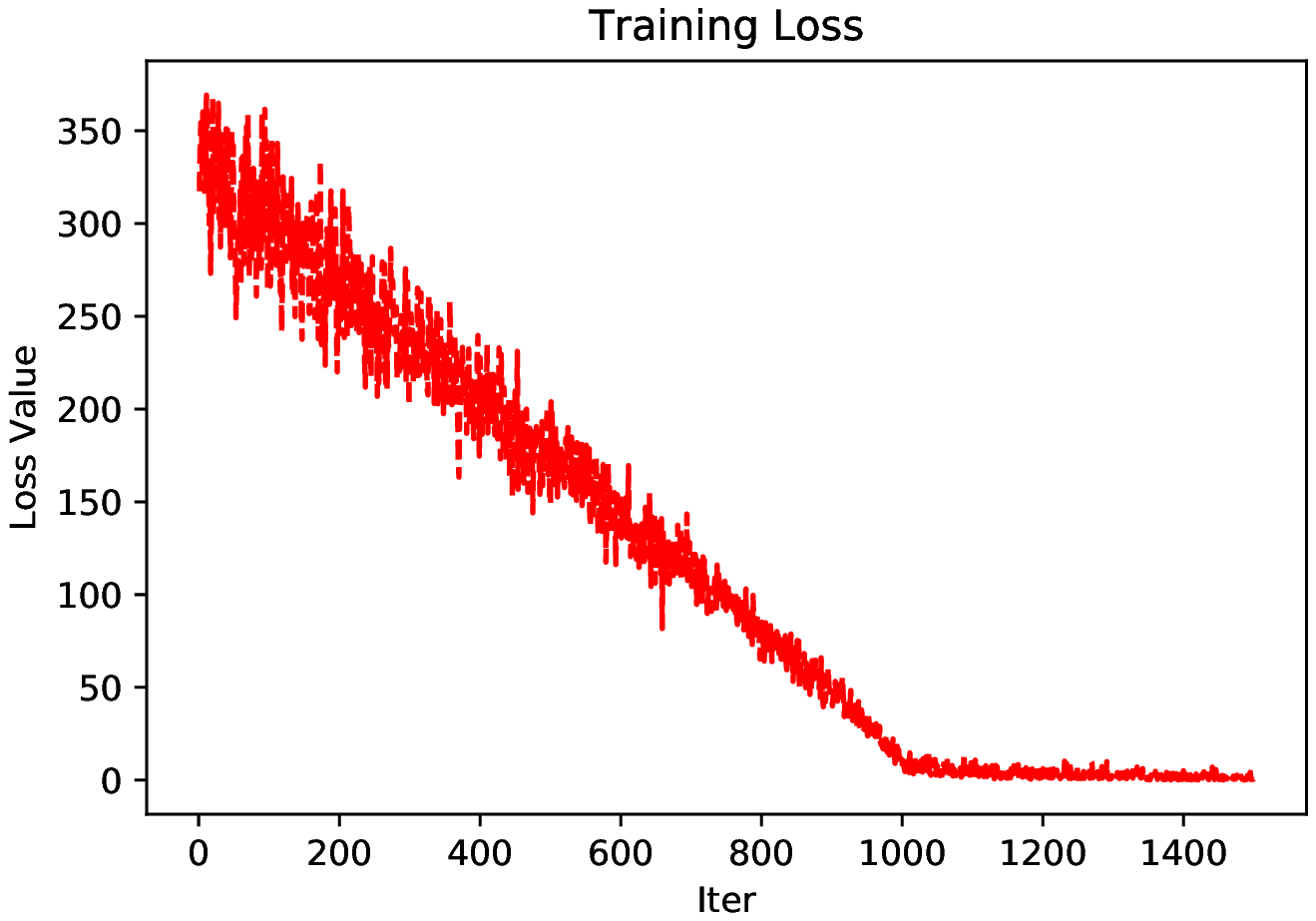}
}
\subfigure[Training Loss on Breastcancer.] { 
\includegraphics[width=0.28\columnwidth]{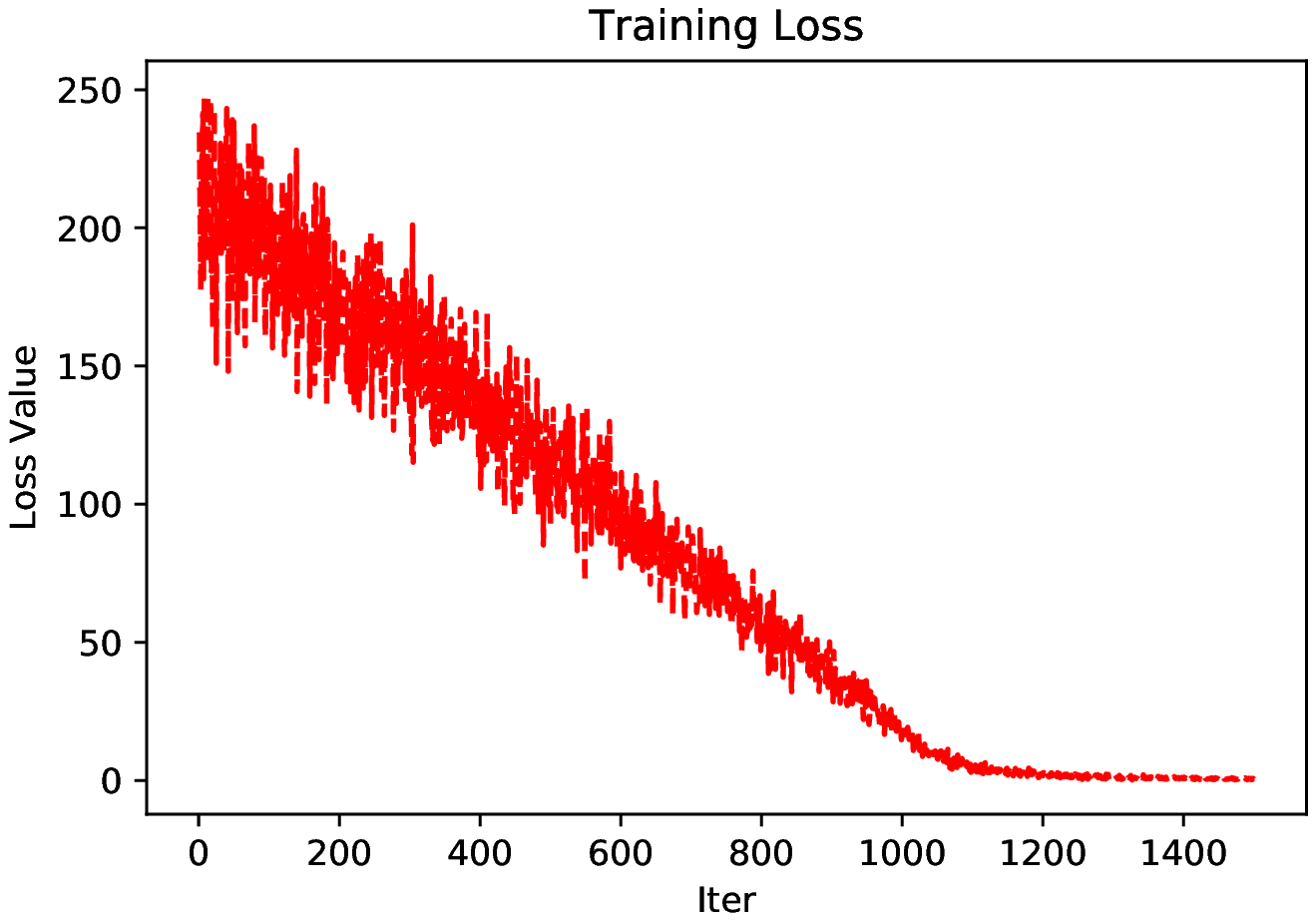}
}
\caption{The UNTIE's Training Loss on Different Data Sets. The stochastic optimization method for UNTIE is Adam \cite{kingma2014adam} with the initial learning rate $10^{-3}$ and the batch size $20$. The X-axis refers to the number of iterations, and y-axis refers to the loss value of UNITE's objective function Eq. \eqref{eq:omega}.}
\label{fig:convergence}
\end{figure}

\begin{figure*}[hbtp] \centering
\subfigure[Time Cost w.r.t. Number of Objects.] { \label{fig:object_time}
\includegraphics[width=0.45\columnwidth]{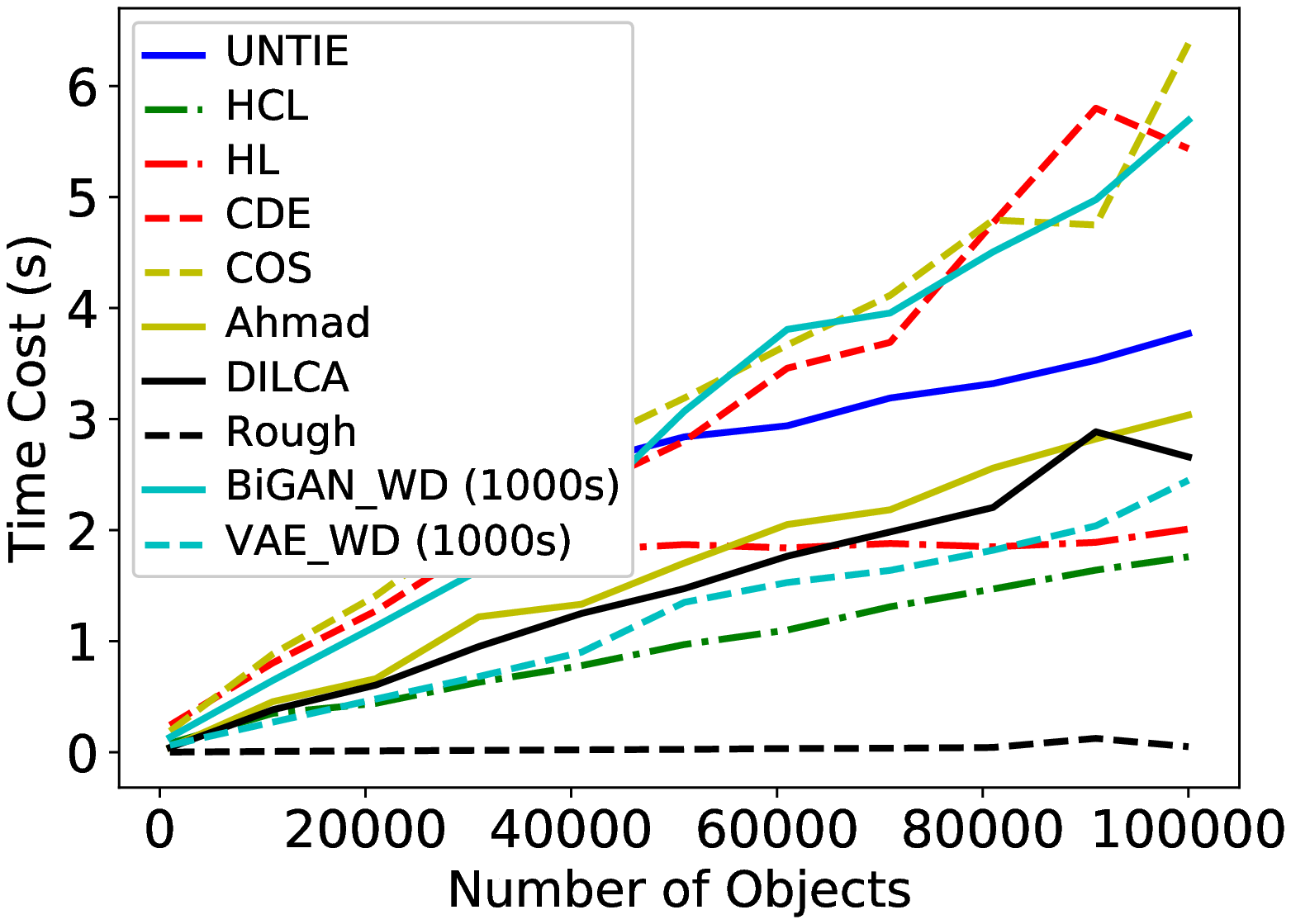}
}
\subfigure[Time Cost w.r.t. Number of Attributes.] { \label{fig:attribute_time}
\includegraphics[width=0.45\columnwidth]{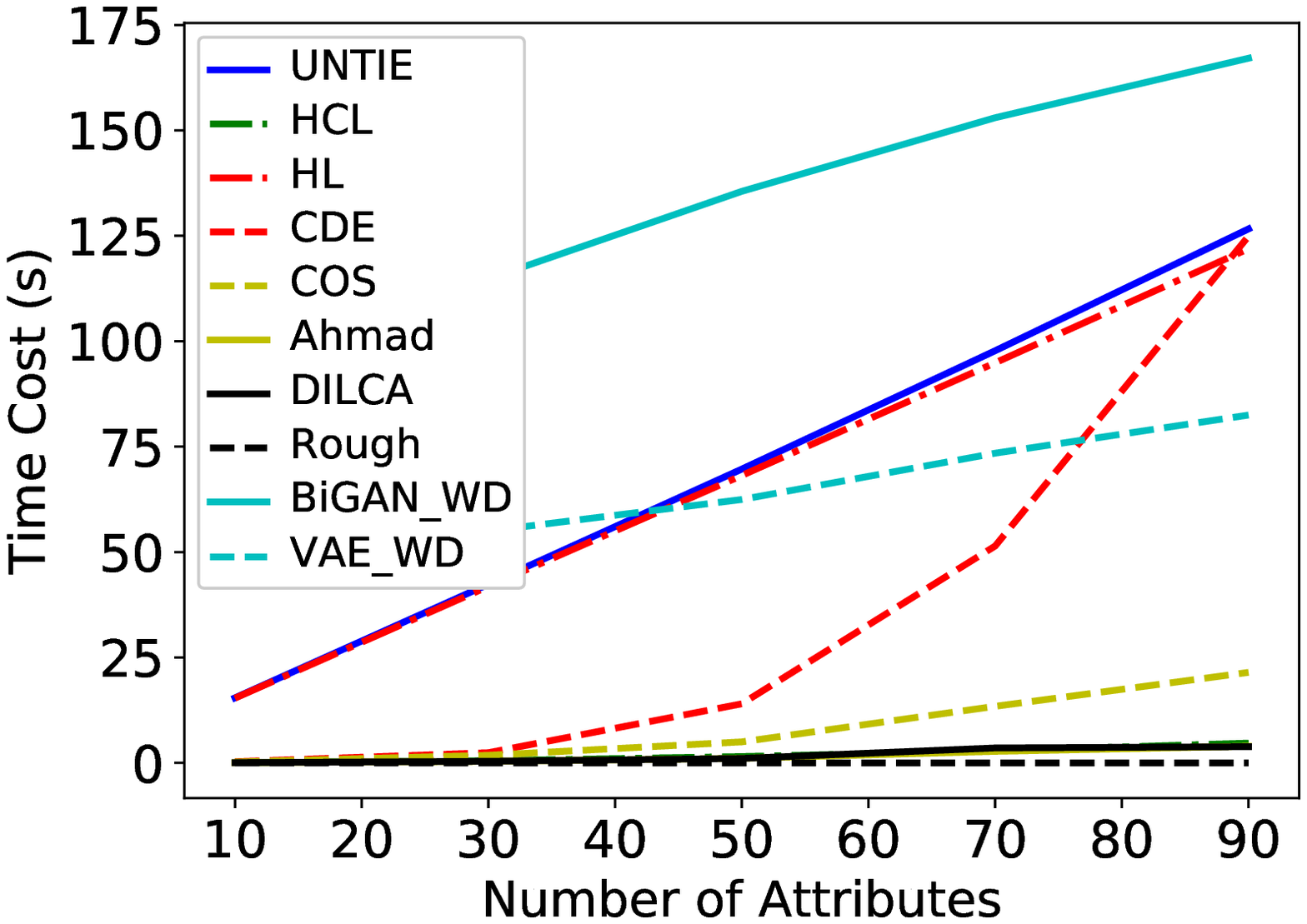}
}
\subfigure[Time Cost w.r.t. Number of Attribute Values.] { \label{fig:value_time}
\includegraphics[width=0.45\columnwidth]{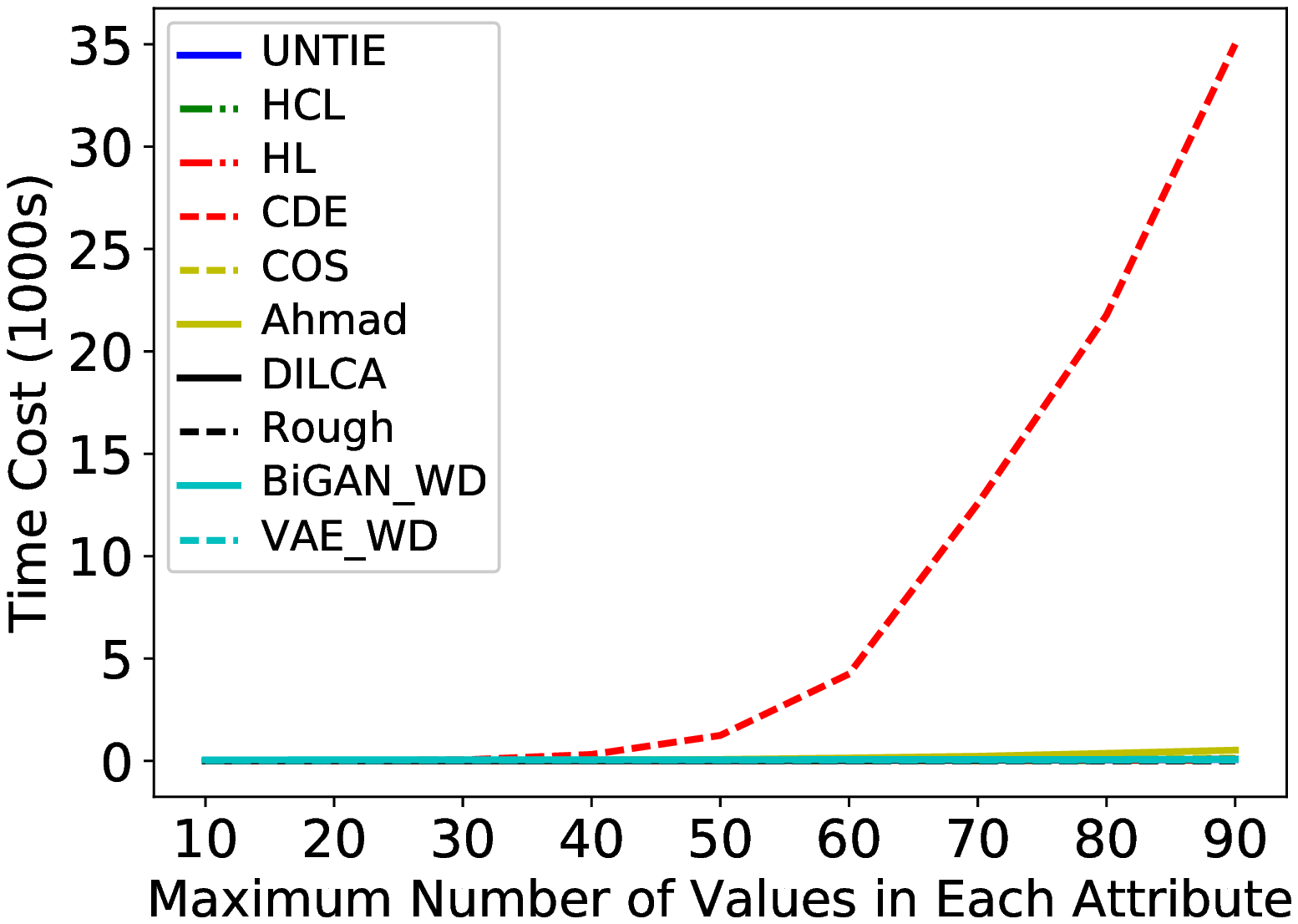}
}
\subfigure[Time Cost w.r.t. Number of Kernels.] {\label{fig:kernelTimeCost}
\includegraphics[width=0.45\columnwidth]{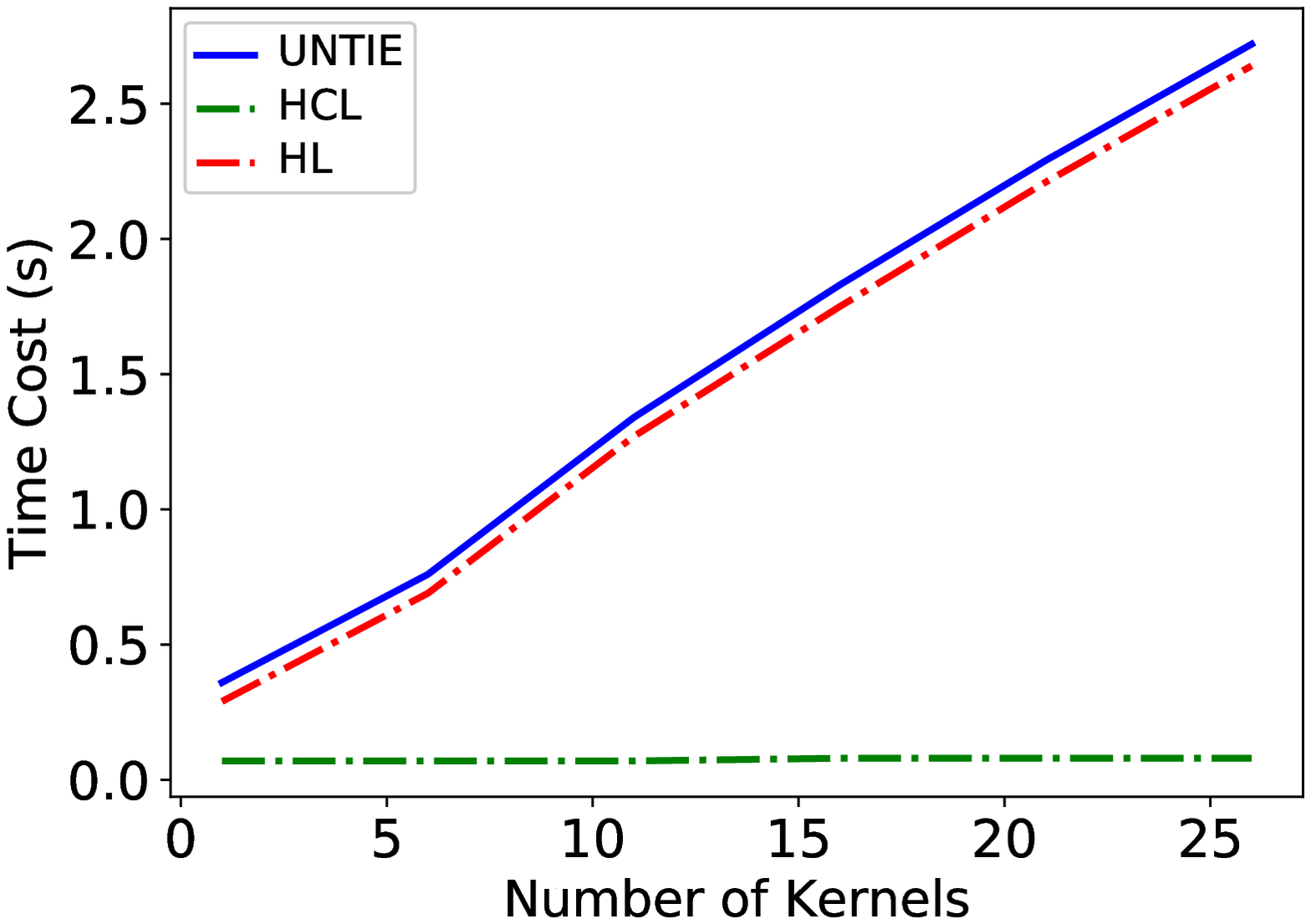}
}

\caption{The UNTIE Computational Cost w.r.t. Data Factors: Object Number $n_o$,  Attribute Number $n_a$, and Maximum Number of Attribute Values $n_{mv}$. The solid line refers to the total time cost of UNTIE. The dotted line refers to the time cost of building the coupling spaces. The star line refers to the time cost of the heterogeneity learning.}
\label{fig:timeCost}
\end{figure*}

\subsection{Evaluating the UNTIE Effectiveness}

\subsubsection{UNTIE-enabled Clustering}

The vector-based representations learned by methods UNTIE, CDE, BiGAN\_WD and VAE\_WD are incorporated into K-means, which is probably the most popular clustering method and is sensitive to a distance measure. The representations learned by similarity-based representation methods, including COS, DILCA, Ahmda, Rough and Hamming, are fed into K-modes, which is the most commonly used clustering method for categorical data. To evaluate the performance of heterogeneity learning, we compare UNTIE with its variant which concatenates the representations in the coupling spaces without heterogeneity learning, denoting as Couplings, which is incorporated into K-means. 

The comparison results are shown in Table \ref{tab:clustering}. The best results are highlighted in bold and $\Delta$ is the ratio of UNTIE's improvement over the best results of other measures. On half of the data sets, UNTIE performs significantly better than the compared methods. For example, the F-score improves 51.72\% on DNAnominal and 30.75\% on Dermatology compared to the best-performing methods DILCA and COS. On the other half of the data sets, UNTIE achieves the same or comparable results to other methods. For example, the F-scores of UNTIE and CDE are both 56.56\% on Mofn3710 and 54.8\% on Tictactoe. 
UNTIE effectively captures the intrinsic categorical data characteristics by revealing the value and attribute couplings and the heterogeneity within and between these couplings to induce the representation. These embedded characteristics guarantee the effectiveness of UNTIE representations, ensuring that the UNTIE-enabled clustering can generally achieve better results than others. It is noted that BiGAN\_WD, and VAE\_WD perform badly on most of data sets except Mofn3710 where BiGAN\_WD achieves the highest.

\begin{table*}[htbp]
  \centering
  \footnotesize
  \caption{Clustering F-Score ($\%$) with Different Embedding Methods: The value-based representations are fed into k-means and the similarity-based representations are fed into k-modes to get the clustering results. The best results are highlighted in bold. $\Delta$ indicates the UNTIE's improvement over the best results of the other measures. The averaged rank of a method over all data sets with significant difference from others w.r.t. the Bonferroni-Dunn test (p-value $<$ 0.1) is labelled by $^{*}$.
}
    \begin{tabular}{l|ll|llllllll|l}
    \toprule
    Dataset & UNTIE & Couplings & CDE   & COS   & Ahmad  & DILCA & Rough & Hamming & BiGAN\_WD & VAE\_WD & $\Delta$ \\
    \midrule
    Zoo  &  \textbf{76.12}  &  74.85  &  75.04  &  72.10  &  71.34  &  71.34  &  62.79  &  73.27  &  56.93  &  24.41  &  1.44$\%$ \\
DNAPromoter  &  \textbf{95.28}  &  92.45  &  61.61  &  49.24  &  49.92  &  85.85  &  63.20  &  52.68  &  51.99  &  50.87  &  10.98$\%$ \\
Hayesroth  &  \textbf{54.17}  &  \textbf{54.17}  &  52.85  &  38.98  &  33.76  &  32.87  &  38.92  &  33.06  &  44.91  &  37.14  &  2.50$\%$ \\
Hepatitis  &  70.40  &  \textbf{73.64}  &  69.82  &  46.29  &  66.72  &  65.13  &  59.21  &  59.21  &  61.08  &  51.24  &  0.00$\%$ \\
Audiology  &  34.99  &  34.48  &  32.18  &  27.71  &  \textbf{35.38}  &  31.77  &  22.36  &  29.05  &  20.00  &  19.97  &  0.00$\%$ \\
Housevotes  &  \textbf{90.51}  &  88.36  &  89.65  &  88.36  &  88.36  &  88.79  &  87.04  &  86.64  &  83.64  &  53.84  &  0.96$\%$ \\
Spect  &  55.04  &  55.04  &  52.55  &  36.26  &  34.93  &  34.76  &  \textbf{57.63}  &  35.94  &  34.71  &  48.38  &  0.00$\%$ \\
Mofn3710  &  56.65  &  44.69  &  56.65  &  50.18  &  50.22  &  48.68  &  50.62  &  50.98  &  \textbf{60.34}  &  49.00  &  0.00$\%$ \\
Soybeanlarge  &  \textbf{69.29}  &  64.88  &  62.19  &  60.10  &  56.84  &  59.42  &  46.41  &  55.31  &  48.38  &  14.83  &  11.42$\%$ \\
Primarytumor  &  24.62  &  24.87  &  23.43  &  19.81  &  23.65  &  21.76  &  22.38  &  \textbf{26.19}  &  22.17  &  14.68  &  0.00$\%$ \\
Dermatology  &  \textbf{97.51}  &  72.78  &  73.10  &  74.58  &  72.87  &  72.61  &  57.99  &  66.60  &  38.54  &  23.82  &  30.75$\%$ \\
ThreeOf9  &  34.86  &  34.86  &  54.63  &  35.32  &  35.32  &  35.32  &  \textbf{65.19}  &  54.22  &  50.03  &  54.64  &  0.00$\%$ \\
Wisconsin  &  93.91  &  95.58  &  \textbf{96.20}  &  94.28  &  95.12  &  95.49  &  94.44  &  89.98  &  74.26  &  81.45  &  0.00$\%$ \\
Crx  &  \textbf{85.49}  &  52.65  &  52.65  &  36.99  &  52.65  &  79.29  &  63.47  &  79.29  &  51.81  &  51.69  &  7.82$\%$ \\
Breastcancer  &  93.27  &  94.75  &  95.20  &  93.56  &  94.89  &  \textbf{95.25}  &  94.37  &  93.27  &  65.94  &  79.15  &  0.00$\%$ \\
Mammographic  &  82.77  &  \textbf{82.89}  &  81.66  &  80.06  &  81.66  &  82.65  &  80.67  &  81.50  &  60.48  &  70.59  &  0.00$\%$ \\
Tictactoe  &  54.80  &  \textbf{62.61}  &  54.80  &  51.88  &  50.87  &  52.97  &  50.19  &  53.59  &  54.38  &  50.24  &  0.00$\%$ \\
Flare  &  37.08  &  31.20  &  32.44  &  35.79  &  34.20  &  35.59  &  38.85  &  \textbf{39.22}  &  31.98  &  22.30  &  0.00$\%$ \\
Titanic  &  33.72  &  29.77  &  33.72  &  29.77  &  33.72  &  33.72  &  \textbf{36.27}  &  33.72  &  31.58  &  28.61  &  0.00$\%$ \\
DNAnominal  &  \textbf{89.79}  &  67.70  &  51.14  &  41.91  &  46.68  &  59.18  &  43.28  &  41.44  &  35.18  &  32.21  &  51.72$\%$ \\
Splice  &  79.73  &  42.29  &  \textbf{87.12}  &  31.31  &  47.34  &  45.87  &  42.79  &  42.48  &  26.60  &  32.55  &  0.00$\%$ \\
Krvskp  &  51.09  &  51.09  &  51.03  &  46.72  &  \textbf{55.17}  &  \textbf{55.17}  &  53.73  &  53.86  &  42.94  &  50.36  &  0.00$\%$ \\
Led24  &  \textbf{69.50}  &  45.82  &  48.03  &  53.91  &  51.83  &  61.08  &  32.65  &  28.82  &  18.38  &  13.12  &  13.79$\%$ \\
Mushroom  &  82.69  &  82.76  &  82.83  &  \textbf{82.91}  &  82.86  &  82.39  &  78.18  &  82.29  &  71.48  &  60.78  &  0.00$\%$ \\
Connect4  &  \textbf{33.20}  &  31.14  &  31.91  &  27.23  &  32.88  &  33.14  &  30.34  &  31.43  &  30.53  &  29.18  &  0.18 $\%$  \\
\midrule
Averaged Rank$^{*}$ & \textbf{2.82} & 4.34 & 3.62 & 6.62 & 4.9 & 4.78 & 5.7 & 5.66 & 7.8 & 8.76 & 0.8 \\
    \bottomrule
    \end{tabular}%
  \label{tab:clustering}%
\end{table*}%

To statistically compare UNTIE's performance with the above categorical representation methods, we calculate their averaged ranks by the Friedman test and Bonferroni-Dunn test \cite{demvsar2006statistical}. The $\chi_F^2$ of Friedman test is 83.21 associated with $p$-value $3.71e^{-14}$. This result indicates that the performance of all the compared methods is not equal. Further, the Bonferroni-Dunn test evaluates the critical difference (CD) between UNTIE and other methods, and shows the CD at $p$-value $< 0.1$ is 2.17. As shown in Table \ref{tab:clustering}, UNTIE achieves an overall averaged rank 2.82, which is better than other measures. For example, it is 0.8 better than that of the best state-of-the-art method CDE (3.62), and 5.94 better than VAE\_WD (8.76). Regarding the CD, the UNTIE's performance is significantly better than most of state-of-the-art methods except CDE, DILCA, and Adam. 
Although UNTIE and CDE do not show significant difference under the Bonferroni-Dunn test at $p$-value $< 0.1$, UNTIE captures the heterogeneity in couplings which cannot be learned by CDE. Therefore, the performance of UNTIE is better than CDE in most cases, especially on data sets with complex structures and heterogeneous distributions. For example, on DNAnominal, UNTIE achieves $89.79\%$ while CDE only achieves $51.14\%$ in terms of F-score.
All the comparison results are shown in Fig. \ref{fig:CDDiagram}, which reveals UNTIE is significantly ($p < 0.1$) better than almost all the compared categorical representation methods.

\begin{figure}[!hbtp] \centering
\includegraphics[width=0.75\columnwidth]{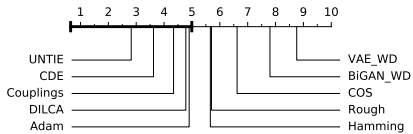}
\caption{Comparison of UNTIE vs. the Other Representation Methods per the Bonferroni-Dunn Test. All methods with ranks outside the marked interval are significantly different (p $<$ 0.1) from UNTIE.}
\label{fig:CDDiagram}
\end{figure}

The results also show that UNTIE and Couplings achieve the overall performance of 2.82 and 4.34 averaged rank, respectively, in comparison with 7.8 and 8.76 for BiGAN\_WD and VAE\_WD which rank the worst. This shows that heterogeneity learning contributes to an additional 1.52 averaged rank over the representation of Couplings. UNTIE does not consistently beat Couplings over all data sets, showing that not all data sets involve strong heterogeneity. For example, on Hepatitis, Mammographic and Tictactoe, the couplings-enabled representations show better results, while both UNTIE and Couplings do not make significant improvement over other methods, which also demonstrates that the clustering labels in these data sets are not sensitive to the captured couplings and heterogeneity. This may indicate other unknown complexities in these data sets that could be further explored.

\subsubsection{UNTIE-enabled Retrieval}

\begin{figure*}[!hbtp] \centering
\subfigure[Precision on Dermatology.] { \label{fig:apk}
\includegraphics[width=0.6\columnwidth]{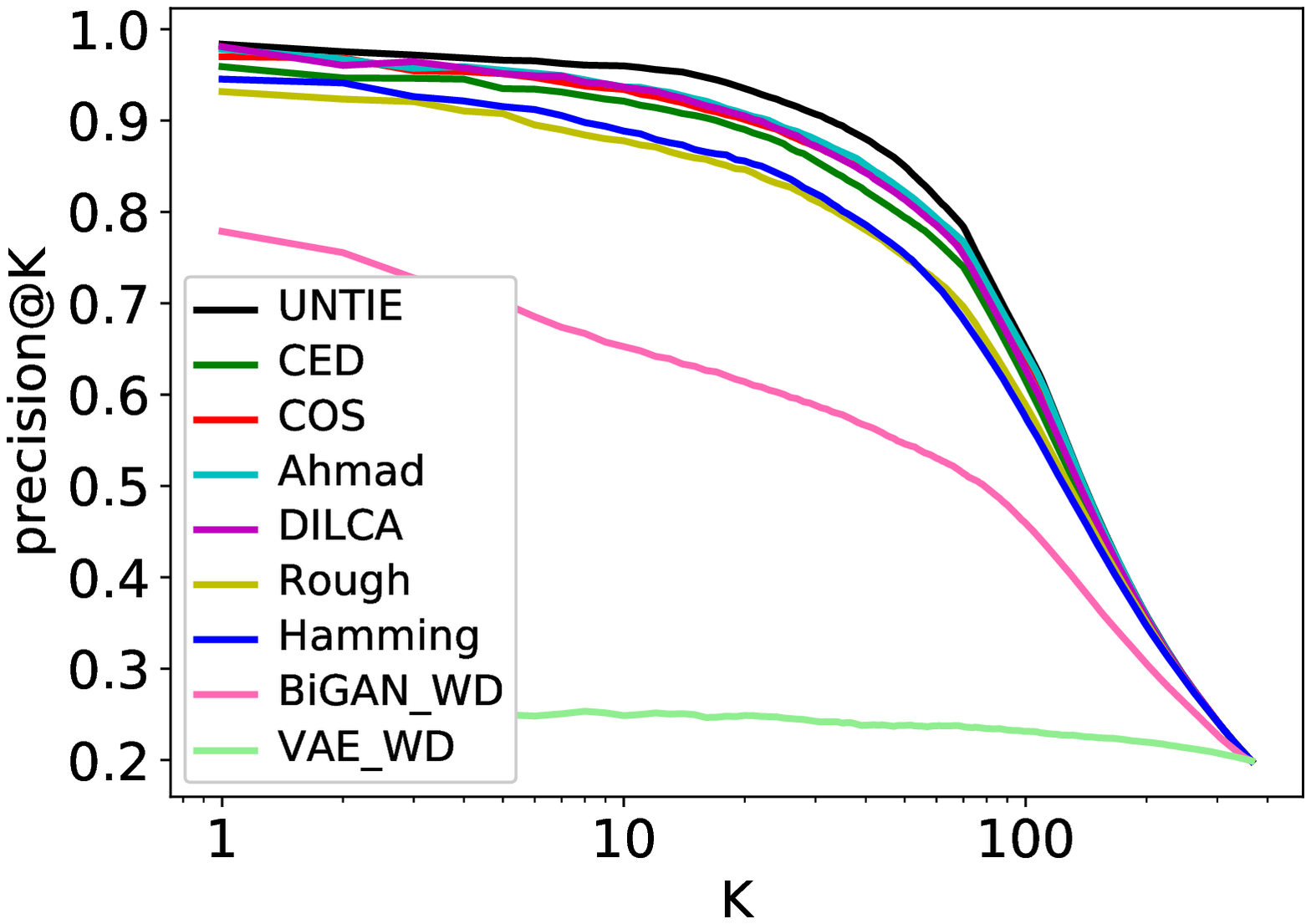}
}
\subfigure[Precision on DNAnominal.] { \label{fig:bpk}
\includegraphics[width=0.6\columnwidth]{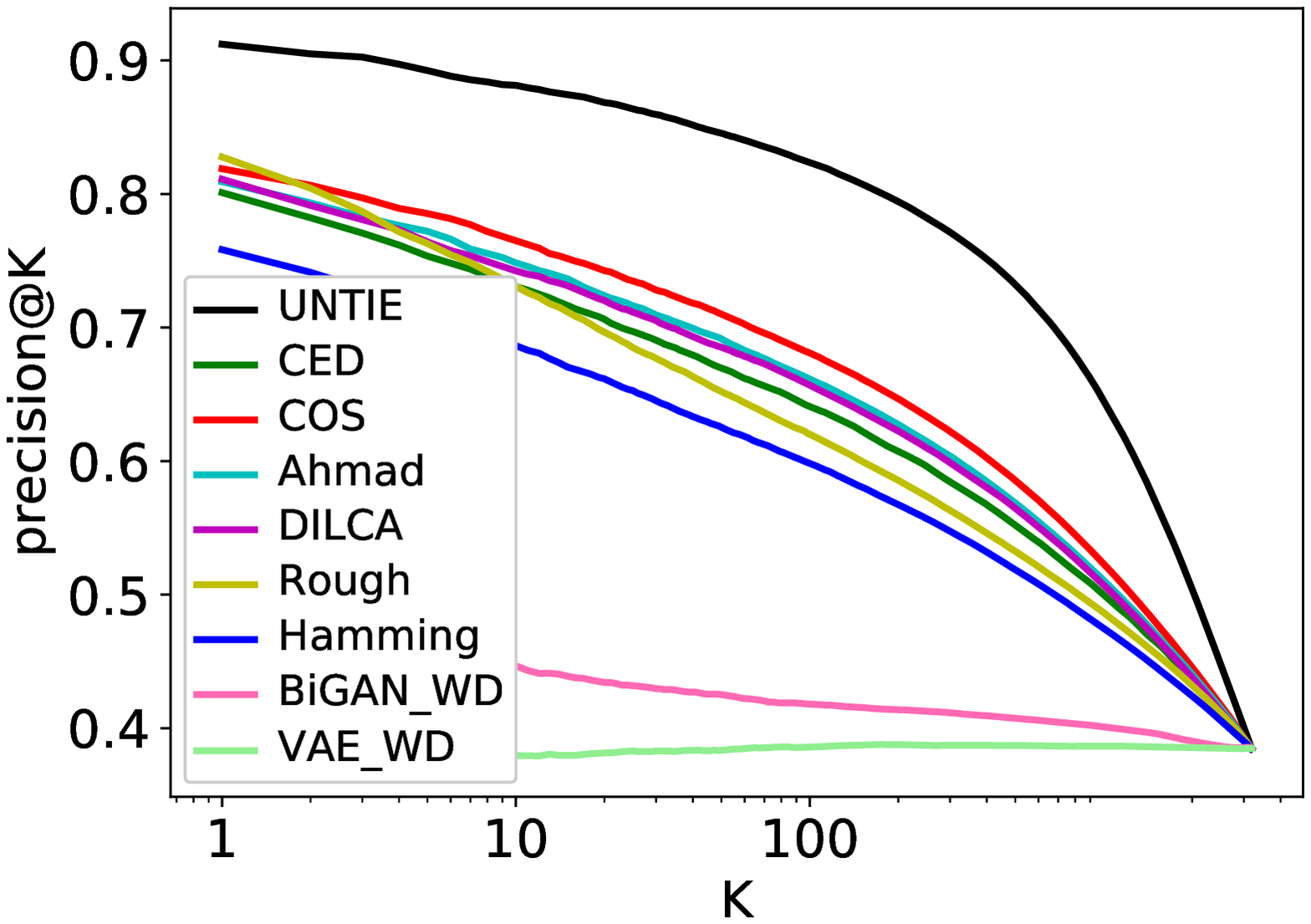}
}
\subfigure[Precision on Splice.] { \label{fig:cpk}
\includegraphics[width=0.6\columnwidth]{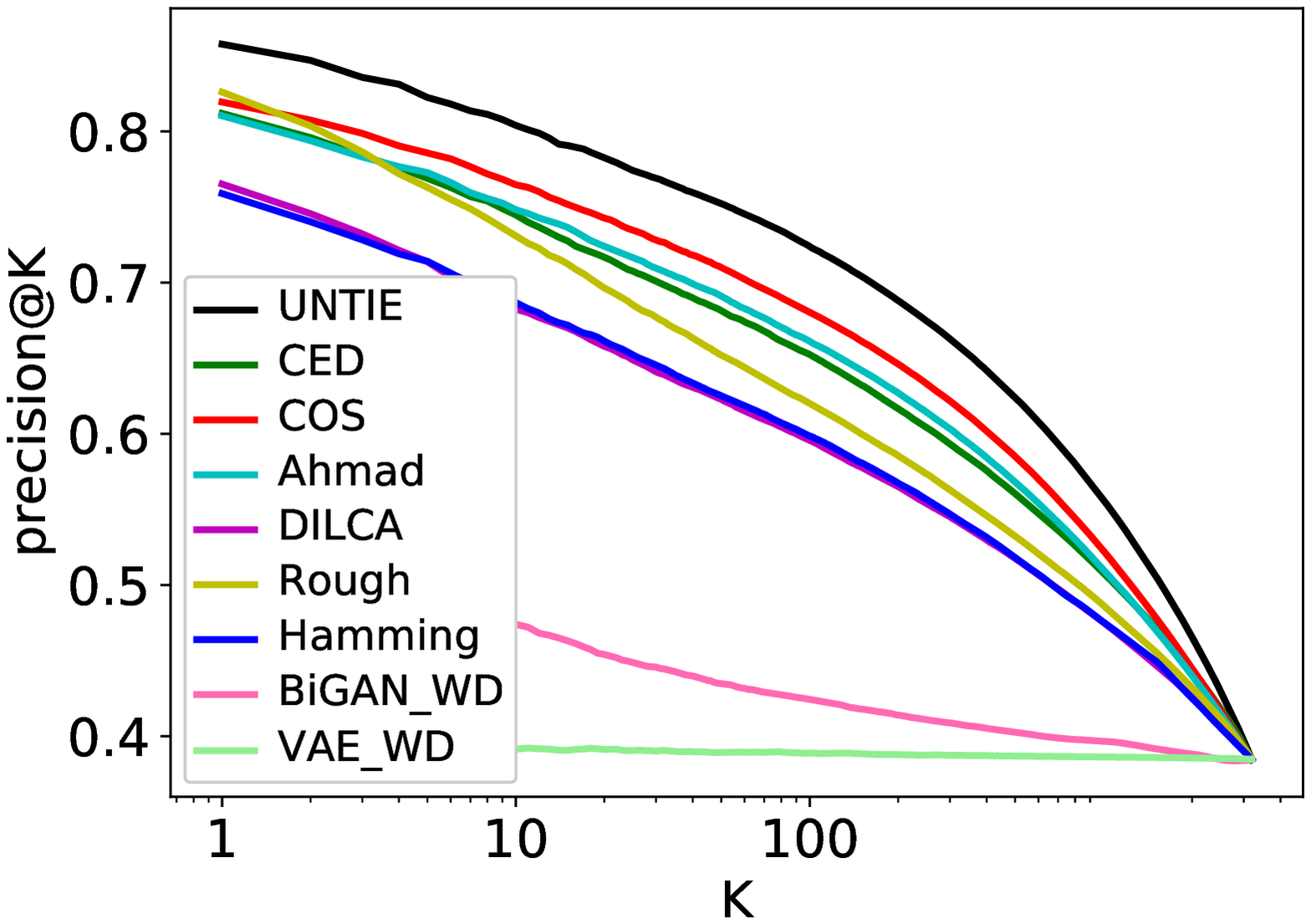}
}
\caption{The Precision@k of Different Categorical Data Representation Methods: A better metric yields a higher value.}
\label{fig:retrieval}
\end{figure*}

We further test the UNTIE representation performance of object retrieval, which also heavily depends on data representation. Every object is used as a query, and its \textit{k}-closest objects are retrieved per a distance measure. The precision@\textit{k}, i.e., the fraction of the retrieved \textit{k} objects that are the same-class neighbors, is reported. We use the Euclidean distance for UNTIE-, CDE-, BiGAN\_WD- and VAE\_WD-represented data to compare with the distance measured by COS, DILCA, Ahmda, Rough and Hamming for retrieval. Three data sets Dermatology, DNAnominal and Splice are tested to evaluate the UNTIE-enabled retrieval performance.

Different from the clustering results, the precision@\textit{k} of retrieval can demonstrate the quality of learned representation from local (when \textit{k} is small) to global (when \textit{k} is large). The results are shown in Fig. \ref{fig:retrieval}, in which
the precision of UNTIE-enabled retrieval consistently outperforms the others. It reflects that UNTIE can capture more details of data distributions than other representation methods, which is powered by learning hierarchical value-to-attribute couplings  and heterogeneities.

\subsection{Evaluating the UNTIE Efficiency}

The efficiency of UNTIE is affected by the number of iterations to achieve convergence in Algorithm \ref{algorithm} and different data factors. In this section, we first empirically evaluate the convergence speed of UNTIE, and then evaluate the computational cost of UNTIE under different data factors. 

\subsubsection{The UNTIE Convergence}

Due to space limitation, we randomly select six real data sets to demonstrate the convergence of UNTIE. The optimization method for UNTIE is Adam with the same setting as in Section \ref{sec:parameters}. The training loss of objective function Eq. \eqref{eq:omega} on these data sets is shown in Fig. \ref{fig:convergence}, the loss value converges rapidly at around 1,000 iterations. Since the batch size in each iteration is only $20$, the time cost of 1,000 iterations is very low. 

\subsubsection{The UNTIE Computational Cost w.r.t. Data Factors}

We further generate synthetic data to evaluate the computational cost of UNTIE in terms of the following data factors \cite{dst_Cao15}: the number of objects $n_o$, the number of attributes $n_a$, and the maximum number of values in each attribute $n_{mv}$. The default settings of these factors are as follows: $n_o$ is $1,000$, $n_a$ is $10$, and $n_{mv}$ is $3$. We generate three groups of data and tune one of these factors for each group. For the first group of data, the number of objects is adjusted from $1,000$ to $100,000$. For the second group of data, the number of attributes is tuned from $10$ to $100$. For the third group of data, the maximum number of values in attributes is changed from $10$ to $100$. The time cost of UNTIE under each data factor is shown in Fig. \ref{fig:timeCost}. 

Fig. \ref{fig:object_time} shows that the time cost of UNTIE is almost stable (from $1.8$(s) to $3.6$(s)), which demonstrates it has good scalability w.r.t. the amount of data $n_o$. Our analysis shows that the minor time cost increase is caused by the Python built-in functions when identifying categorical value location in the data. Since this cost increases with an extremely small proportion w.r.t. the amount of data, we can ignore it when applying UNTIE. Fig. \ref{fig:attribute_time} and Fig. \ref{fig:value_time} demonstrate the time cost has approximately linear relation with both $n_a$ and $n_{mv}$, which is consistent with the time complexity of UNTIE analyzed in Section \ref{sec:complexity}. These results also show that the main cost of UNTIE is on the heterogeneity learning (HL), which has the linear relation with both $n_a$ and $n_{mv}$. Meanwhile, the cost of building hierarchical coupling learning (HCL) has quadratic relation with $n_a$ and $n_{mv}$. The reason is that UNTIE calculates the pairwise value relations when learning inter-attribute couplings. However, it only slightly affects the cost of UNTIE when $n_a$ and $n_{mv}$ are small. For categorical data with high dimensionality, a trade-off between sufficiently capturing couplings and preserving efficiency is required.

The computational costs of UNTIE and state-of-the-art methods are at the same level in terms of data factor $n_o$. It indicates all of these methods can handle large amount of data. We can see that UNTIE has a higher computational cost compared to other methods in terms of $n_a$. As shown in \ref{fig:attribute_time}, the higher cost is brought by heterogeneity learning, which has linear relation with $n_a$. For the hierarchical coupling learning, the cost of UNTIE is at the same level as the state-of-the-art methods. As for $n_{mv}$, UNTIE is much more efficient than CDE, which is the state-of-the-art method with the best representation performance as shown in the previous experiments.

To evaluate the relation between time cost and the number of kernel functions, we set the number of kernels used in UNTIE from 1 to 30 and test the computational cost of UNTIE on the synthetic data set with the default data factors. The UNTIE time cost with a different number of kernels is shown in Fig. \ref{fig:kernelTimeCost}. This shows that the UNTIE time cost is linear to the number of kernels with a very small slope. Increasing the number of kernels only slightly affects the computational time of UNTIE. This is consistent with our theoretical analysis, which indicates $n_{\bm{\omega}}$ is linear to the time complexity of UNTIE. Here, $n_{\bm{\omega}}$ has a linear relation with the number of kernels. 

\begin{table*}[!htbp]
	\centering
	\footnotesize
	\caption{KNN, SVM, RF and LR Classification F-score (\%) with UNTIE and CDE. The best results are highlighted in bold.}
	\begin{tabular}{l|ll|ll|ll|ll}
		\toprule
		Data set & UNTIE-SVM & CDE-SVM & UNTIE-KNN & CDE-KNN & UNTIE-RF & CDE-RF & UNTIE-LR & CDE-LR \\
		\midrule
		zoo   & \textbf{100} & 88.00$\pm$18.33 & \textbf{100} & \textbf{100} & \textbf{100} & \textbf{100} & \textbf{100} & \textbf{100} \\
		DNAPromoter & \textbf{94.42$\pm$6.81} & 91.37$\pm$7.41 & \textbf{87.19$\pm$10.79} & 76.06$\pm$10.62 & \textbf{89.32$\pm$8.81} & 87.58$\pm$11.37 & \textbf{94.41$\pm$6.03} & 90.35$\pm$8.20 \\
		hayesroth & \textbf{82.23$\pm$6.22} & 80.92$\pm$6.96 & 60.15$\pm$10.52 & \textbf{62.38$\pm$10.33} & \textbf{82.48$\pm$9.00} & 82.11$\pm$8.26 & \textbf{82.08$\pm$7.29} & \textbf{82.08$\pm$7.29} \\
		lymphography & \textbf{87.06$\pm$12.02} & 85.68$\pm$11.26 & \textbf{82.22$\pm$11.08} & 79.03$\pm$15.07 & 82.17$\pm$14.59 & \textbf{84.16$\pm$11.31} & \textbf{83.12$\pm$12.90} & 81.67$\pm$13.84 \\
		hepatitis & \textbf{48.13$\pm$8.65} & 46.85$\pm$6.04 & 70.16$\pm$13.37 & \textbf{70.47$\pm$14.52} & \textbf{66.61$\pm$11.12} & 64.30$\pm$12.25 & \textbf{70.39$\pm$16.85} & \textbf{70.39$\pm$16.85} \\
		audiology & \textbf{73.41$\pm$7.29} & 73.25$\pm$5.90 & \textbf{48.19$\pm$9.47} & 47.79$\pm$8.20 & \textbf{63.80$\pm$10.24} & 59.78$\pm$11.83 & 47.70$\pm$7.11 & \textbf{66.29$\pm$11.26} \\
		housevotes & 96.71$\pm$3.88 & \textbf{96.92$\pm$3.69} & \textbf{95.58$\pm$3.69} & 92.54$\pm$4.18 & \textbf{96.02$\pm$4.82} & 95.13$\pm$4.61 & 93.74$\pm$5.00 & \textbf{93.98$\pm$5.15} \\
		spect & 67.60 $\pm$ 11.92 & \textbf{68.46$\pm$10.88} & \textbf{55.41$\pm$9.95} & 50.54$\pm$8.56 & 66.61$\pm$10.93 & \textbf{67.13$\pm$12.08} & \textbf{69.48$\pm$12.52} & \textbf{69.48$\pm$12.52} \\
		mofn3710 & \textbf{100} & \textbf{100} & \textbf{87.18$\pm$6.29} & 86.04$\pm$7.77 & 76.64$\pm$10.51 & \textbf{78.31$\pm$12.16} & \textbf{100} & \textbf{100} \\
		soybeanlarge & 90.94$\pm$3.83 & \textbf{93.61$\pm$4.34} & 93.70$\pm$4.26 & \textbf{96.03$\pm$3.85} & \textbf{92.88$\pm$5.81} & 92.69$\pm$6.12 & \textbf{89.57 $\pm$5.98} & 88.57$\pm$6.97 \\
		\midrule
		Averaged Rank & \textbf{1.35} & 1.65 & \textbf{1.35} & 1.65 & \textbf{1.35} & 1.65 & 1.45 & \textbf{1.55} \\
		\bottomrule
	\end{tabular}%
	\label{tab:flexity}%
\end{table*}%

\subsection{Evaluating the UNTIE Flexibility}

We further demonstrate the UNTIE flexibility when fed into the classifiers KNN, SVM, RF and LR. We randomly select 90\% of objects in each data set for training and the remainder for testing. To reduce the impact of noise and randomness, 20 sampling iterations generate 20 sets of training and test data for the experiments. The averaged classification performance and standard deviation are reported w.r.t. F-score ($\%$). The vector representations learned by UNTIE and CDE are used as the input of these classifiers. The results comparing with CDE-enabled classifiers are shown in Table \ref{tab:flexity} and illustrate that UNTIE representations can fit different classifiers and enhance their performance on categorical data, as compared with the results of CDE-enabled classifiers.

\subsection{Evaluating the UNTIE Stability}

To evaluate the stability of UNTIE per the kernel functions in heterogeneity learning, we adopt three groups of kernel functions where each has a varying number of functions. The first group only contains Gaussian kernels, the second group only contains Polynomial kernels, and the third group mixes Gaussian and Polynomial kernels. The kernel functions in each set are shown in Table \ref{tab:stability}. The clustering F-score enabled by UNTIE w.r.t. different kernel function sets on two data sets DNAPromoter and Monfn3710 are illustrated in Fig.\ref{fig:stability}.

\renewcommand\arraystretch{1.5}
\begin{table}[!htbp]
	\centering
	\footnotesize
	\caption{Three Groups of Kernel Functions for Testing the UNTIE Stability}\label{tab:stability}
	\begin{tabular}{l|l|l}
		\toprule
		Group & Set & Kernel Functions\\
		\midrule
		\multirow{2}{*}{Group 1} & $F_1$ & Gaussian kernels with width $\{2^{-3}, 2^{-2}, \cdots, 2^{3}\}$\\
		\cline{2-3}
        & $F_2$ & Gaussian kernels with width $\{2^{-5}, 2^{-4}, \cdots, 2^{5}\}$\\
        \midrule
        \multirow{2}{*}{Group 2} & $F_3$ & Polynomial Kernels with order $\{1,2\}$\\
        \cline{2-3}
        & $F_4$ & Polynomial Kernels with order $\{1,2,3\}$\\
        \midrule
        \multirow{4}{*}{Group 3} & \multirow{2}{*}{$F_5$} & Gaussian kernels with width $\{2^{-3}, 2^{-2}, \cdots, 2^{3}\}$\\
        & & Polynomial Kernels with order $\{1,2\}$\\
        \cline{2-3}
        & \multirow{2}{*}{$F_6$} & Gaussian kernels with width $\{2^{-5}, 2^{-4}, \cdots, 2^{5}\}$\\
        & & Polynomial Kernels with order $\{1,2,3\}$\\
		\bottomrule
	\end{tabular}
\end{table}

\begin{figure}[!hbtp] \centering
\subfigure[F-Score ($\%$) on DNAPromoter.] {
\includegraphics[width=0.45\columnwidth]{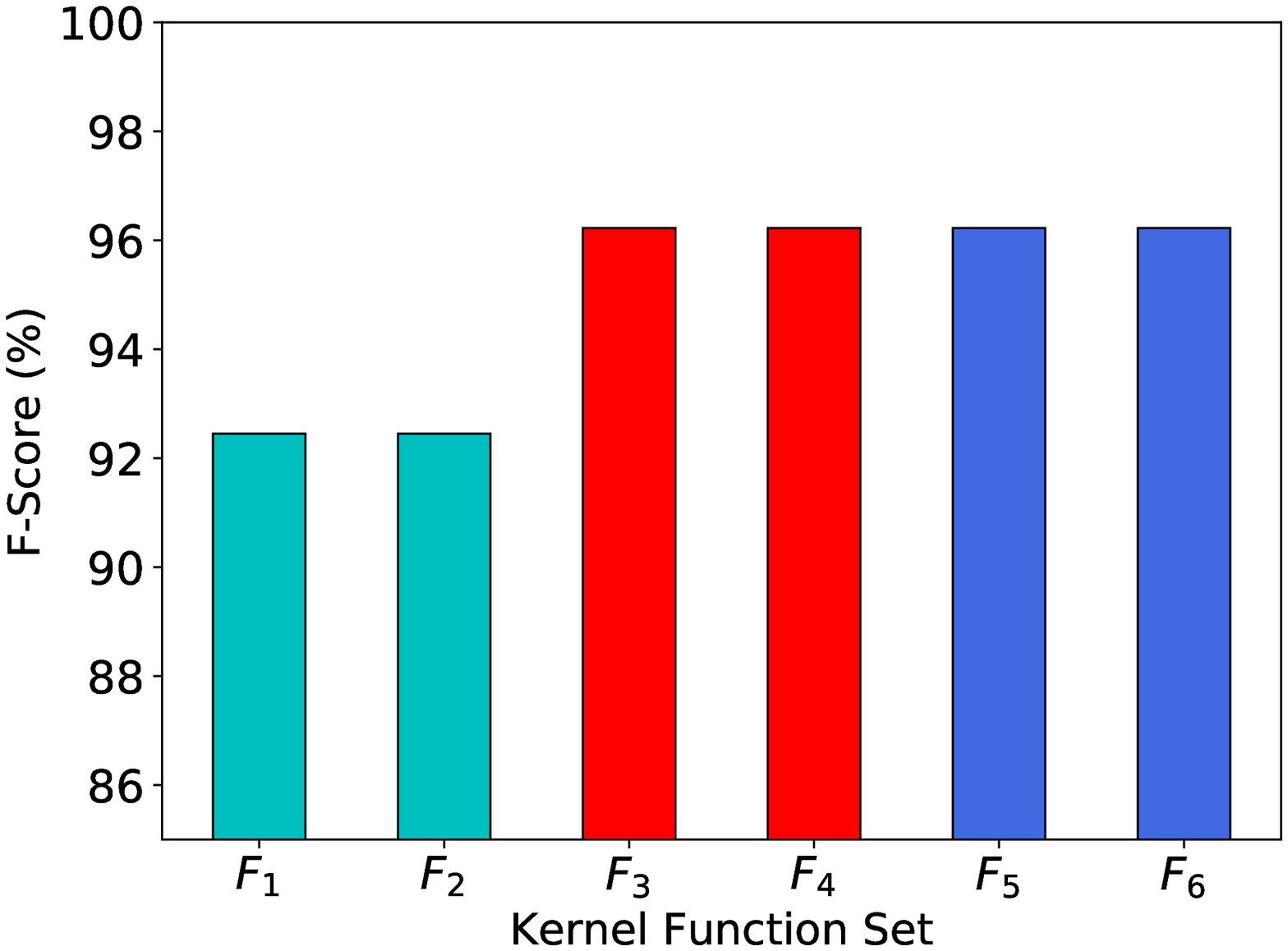}
}
\subfigure[F-Score ($\%$) on Mofn3710.] { 
\includegraphics[width=0.45\columnwidth]{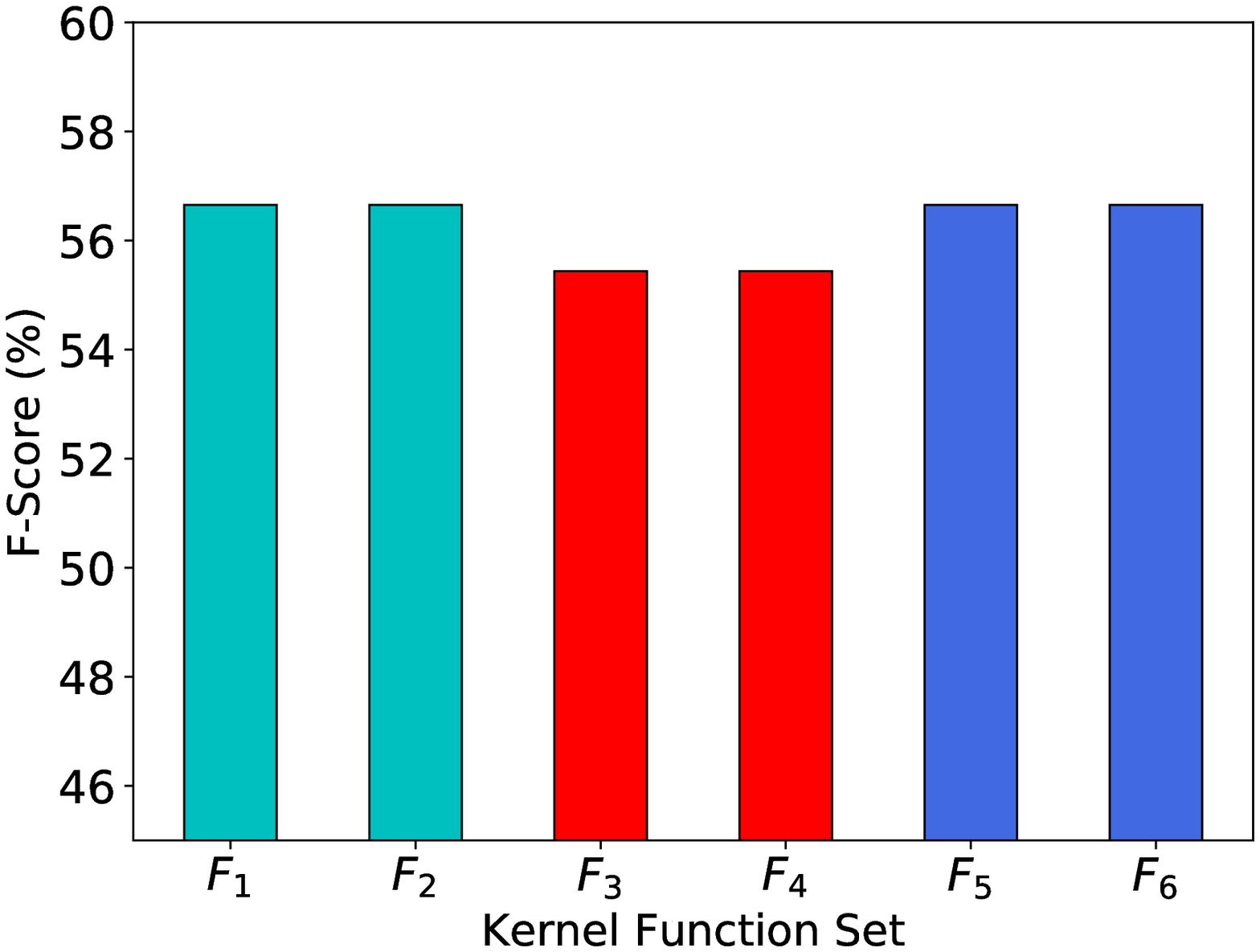}
}
\caption{The Clustering F-Score ($\%$) with UNTIE w.r.t. Different Kernel Function Sets: The same color indicates the same kernel function group.}
\label{fig:stability}
\end{figure}

UNTIE is stable in terms of the number of kernel functions. As shown in Fig. \ref{fig:stability}, UNTIE achieves the same clustering F-score w.r.t. kernel function sets in the same group, where kernel functions are of the same type but different numbers. Although different kernel functions may generate different effects of UNTIE on different data sets (e.g., the Gaussian kernel family in group 1 enables better performance on DNAPromoter, and the Polynomial kernel family in group 2 enables better performance on Mofn3710), UNTIE can comprehensively learn information from multiple kernels while eliminates their redundancy and inconsistency. Accordingly, it always enables the best clustering performance with the kernel function sets $F_5$ and $F_6$ in group 3, which involves kernel functions in both groups 1 and 2.

\section{Conclusions}\label{sec:conclusion}
Categorical data representation is critical yet challenging as complicated coupling relationships and heterogeneities are often embedded in complex categorical values, attributes and objects. Existing work including deep learning is troubled by unsupervised categorical representation learning. This paper introduces a heterogeneous coupling learning method UNTIE for unsupervised categorical data representation. By modeling value-to-object hierarchical couplings and their complementary and inconsistent influence on representations, UNTIE reveals the nonlinear relations between couplings and discloses the heterogeneous distributions within couplings. Both theoretical and empirical analyses show the effectiveness and efficiency of UNTIE. 

An important lesson learned in UNTIE is to select appropriate kernels w.r.t. specific data characteristics and domain knowledge of the underlying problems. 
This work shows the need and potential of shallow learners in handling complex data characteristics in particular couplings, heterogeneities, and inconsistency. The poor results on some data in Table \ref{tab:clustering} also show the challenge and open issues on categorical representation of complex data characteristics even in small data. In addition, modeling more complex couplings, such as very high-dimensional couplings, may represent more complicated relations and interactions embedded in high-dimensional data.

\section{Acknowledgment}
This work is partially sponsored by the Australian Research Council grants DP190101079 and FT190100734.

\bibliographystyle{IEEEtran}

\bibliography{IEEEabrv,IEEEexample}
\vspace{-1cm}
\begin{IEEEbiography}[{\includegraphics[width=1in, height=1.25in, clip, keepaspectratio]{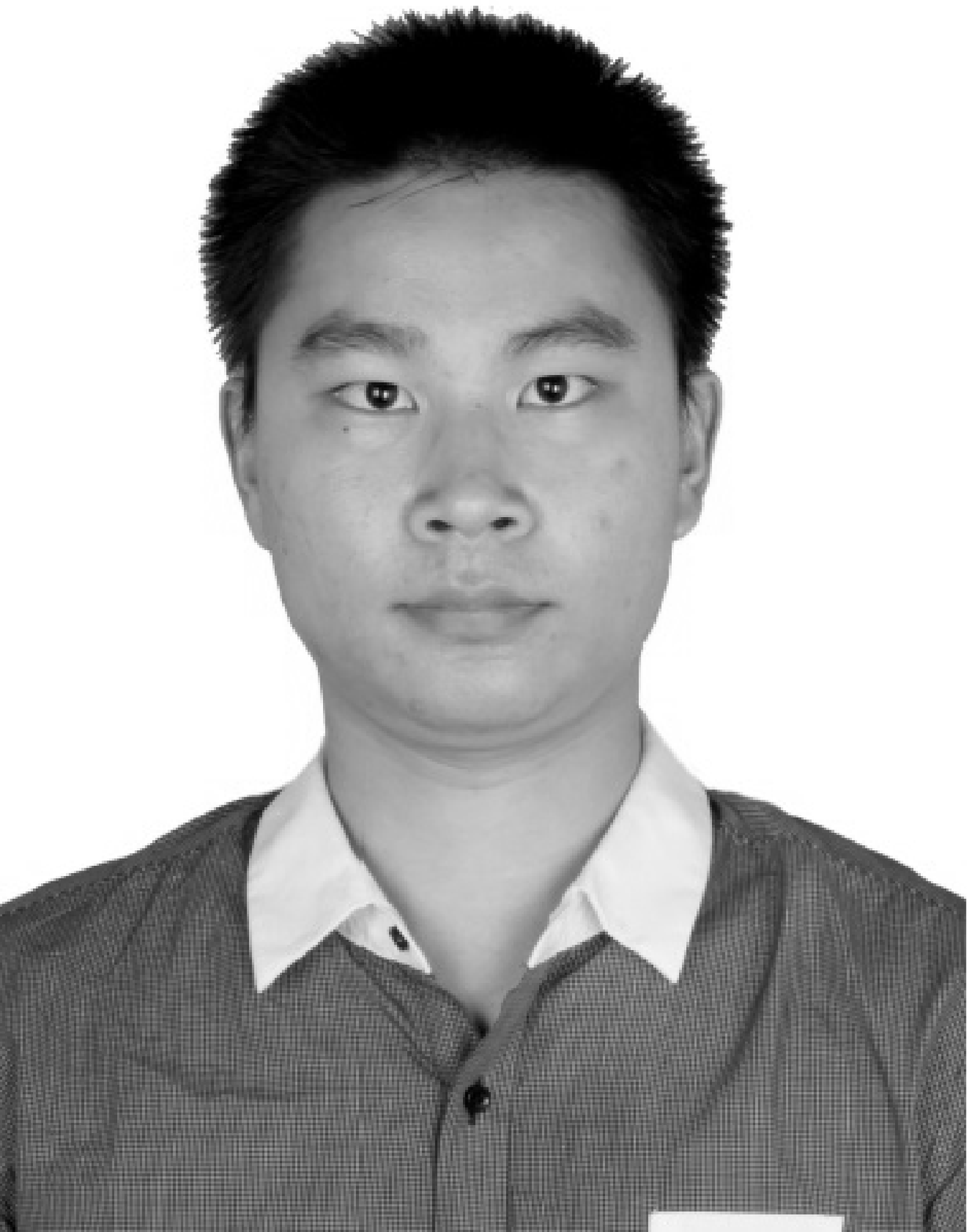}}]{Chengzhang Zhu}
was a PhD student in the Faculty of Engineering and IT, University of Technology Sydney, Australia when this work was done. His research interests include metric learning, non-IID learning, and data representation, in addition to general interest in broad data science especially machine learning.
\end{IEEEbiography}

\vspace{-1cm}
\begin{IEEEbiography}[{\includegraphics[width=1in, height=1.25in, clip, keepaspectratio]{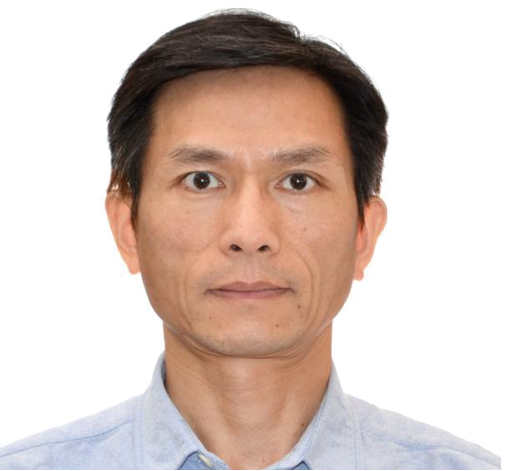}}]{Longbing Cao}
(SM’06) received a PhD in Pattern Recognition and Intelligent Systems from the Chinese Academy of Science and a PhD in Computing Sciences from the University of Technology Sydney. He is a professor and an ARC Future Fellow (Level 3). His current research interests include data science, artificial intelligence, behavior informatics, and their enterprise applications.
\end{IEEEbiography}
\vspace{-1cm}
\begin{IEEEbiography}[{\includegraphics[width=1in, height=1.25in, clip, keepaspectratio]{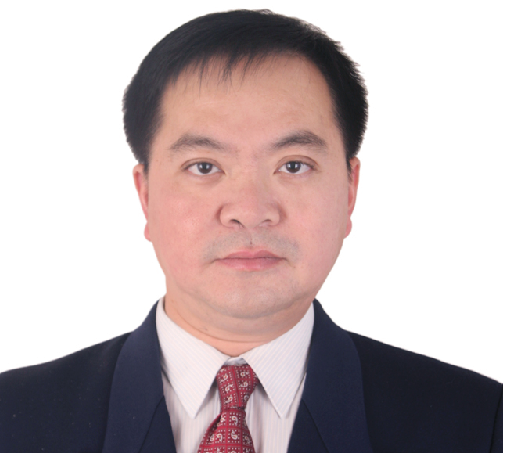}}]{Jianping Yin}
received a PhD in computer science and technology from the National University of Defense Technology in 1990. He is a distinguished professor at the Dongguan University of Technology. 
His research interests include theoretical computer science, artificial intelligence, pattern recognition, and networking algorithms.
\end{IEEEbiography}

\end{document}